\pdfoutput=1

\documentclass{article}

\PassOptionsToPackage{numbers, compress}{natbib}

\usepackage[preprint]{neurips_2025}

\usepackage[utf8]{inputenc} 
\usepackage[T1]{fontenc}    
\usepackage{hyperref}       
\usepackage{url}            
\usepackage{booktabs}       
\usepackage{amsfonts}       
\usepackage{nicefrac}       
\usepackage{microtype}      
\usepackage{xcolor}         
\usepackage{amssymb}
\usepackage{amsmath}
\usepackage{amsthm}

\usepackage{float}
\usepackage[ruled, boxed]{algorithm2e}

\usepackage{wrapfig} 
\usepackage{multirow}
\usepackage{float}
\usepackage{tabularx}
\usepackage{times}
\usepackage{latexsym}
\usepackage[utf8]{inputenc}
\usepackage{booktabs} 
\usepackage{subcaption}
\usepackage{adjustbox}
\usepackage{enumitem}
\usepackage{caption} 
\hypersetup{
    colorlinks=true,        
    linkcolor=darkblue,     
    urlcolor=darkblue       
}

\definecolor{darkblue}{HTML}{000099}

\newtheorem{assumption}{Assumption}
\newtheorem{proposition}{Proposition}

\newtheorem{lemma}{Lemma}
\newtheorem{definition}{Definition}

\newenvironment{restatedproposition}[1]{%
    \renewcommand{\theproposition}{#1}%
    \proposition%
}{%
    \endproposition%
}

\title{Do we really have to filter out random noise in pre-training data for language models?}

\author{%
  Jinghan Ru$^{1}$, Yuxin Xie$^{1}$, Xianwei Zhuang$^{1}$, Yuguo Yin$^{1}$, Zhihui Guo$^{2}$, \\ {\bfseries Zhiming Liu$^{3}$, Qianli Ren$^{4}$, Yuexian Zou$^{1}$\thanks{Corresponding author.}} \\
  $^{1}$ School of Electronic and Computer Engineering, Peking University \\
  $^{2}$ University of Electronic Science and Technology of China \\
  $^{3}$ Hong Kong University of Science and Technology 
  $^{4}$ Sichuan University 
  \\
  \texttt{\{jinghanru, yuxinxie, xwzhuang, ygyin\}@stu.pku.edu.cn,} \\
  \texttt{zhihuiguo@std.uestc.edu.cn, zliugx@connect.ust.hk} \\
  \texttt{qianliren.scu@gmail.com, zouyx@pku.edu.cn} \\
}

\begin{document}
\maketitle
\begin{abstract}
  Web-scale pre-training datasets are the cornerstone of LLMs' success. However, text data curated from the Internet inevitably contains random noise caused by decoding errors or unregulated web content. In contrast to previous works that focus on low quality or synthetic data, our study \textbf{provides the first systematic investigation of such random noise through a cohesive ``What-Why-How'' framework.} Surprisingly, we observed that the resulting increase in the loss of next-token prediction (NTP) was significantly lower than the proportion of random noise even when the model was scaled up to 2.7B. We provide a theoretical justification for this phenomenon, which also elucidates the success of multilingual models and can be applied to multimodal models. On the other hand, experiments show that the model's performance in downstream tasks is not based solely on the NTP loss, which means that random noise may result in degraded downstream performance. To address the potential adverse effects, we introduce a novel plug-and-play Local Gradient Matching loss, which explicitly enhances the denoising capability of the downstream task head by aligning the gradient of normal and perturbed features without requiring knowledge of the model's parameters. Additional experiments on 8 language and 14 vision benchmarks further validate its effectiveness.
\end{abstract}

\section{Introduction}


Large language models (LLMs) have fundamentally transformed the research landscape in natural language processing. The remarkable performance of these autoregressive models is largely attributed to pre-training on extensive datasets, which are gathered by crawling text from the whole Internet. Given the sheer volume of these datasets, they inevitably encompass a wide variety of noise \citep{pretrainguide, wimbd}. 
Consequently, it is imperative to understand its impact, as the quality of pre-training data plays a decisive role in the effectiveness of LLMs \citep{llama}. 
Previous studies \citep{allenphysics, xie2023data} highlight that low quality data can significantly decrease the knowledge capacity and performance of a model, and recursive training of LLMs with synthetic data can lead to model collapse \citep{naturemc, colmmc}. 

However, \textbf{little attention has been paid to the impact of random noise within datasets}. Due to anti-crawling mechanisms, decoding errors, and huge amounts of unmaintained websites, the raw data obtained through web crawling inevitably contains a substantial amount of random noise \citep{webnoise, wavlm, webnoise1}. Although theoretically it may not be challenging to remove such noise, practical limitations in computational resources often result in incomplete data cleaning
\citep{dssurvey, dolma}. 
Therefore, it is of great importance to gain a thorough understanding of the effects of such random noise.

We conduct extensive experiments based on the OpenWebText dataset \citep{openwebtext} used to pre-train language models with the same architecture as GPT-2 and parameter size up to 2.7B. Specifically, to simulate random noise shown in Figure~\ref{fig:1}, we randomly generate a sequence of random integers within the range of 0 to 50256, the vocabulary size of GPT-2's tokenizer, to simulate the tokenization outcome of non-sensical text found on the Internet. 
Interestingly, we observe that the presence of random noise does not lead to a catastrophic failure in model training; instead, its effect on autoregressive loss is disproportionately small, e.g.,~the increase in loss is only about 1\% even with 20\% of the dataset being noisy. 
We provide a theoretical analysis to explain these phenomena, which also sheds light on the success of multilingual models \citep{atri} and multimodal LLMs \citep{vargpt, uniaudio}, indicating the broader implications of studying the random noise.

\begin{wrapfigure}{r}{0.5\columnwidth} 
  \centering
  \vspace{-5mm}
  \includegraphics[width=\linewidth]{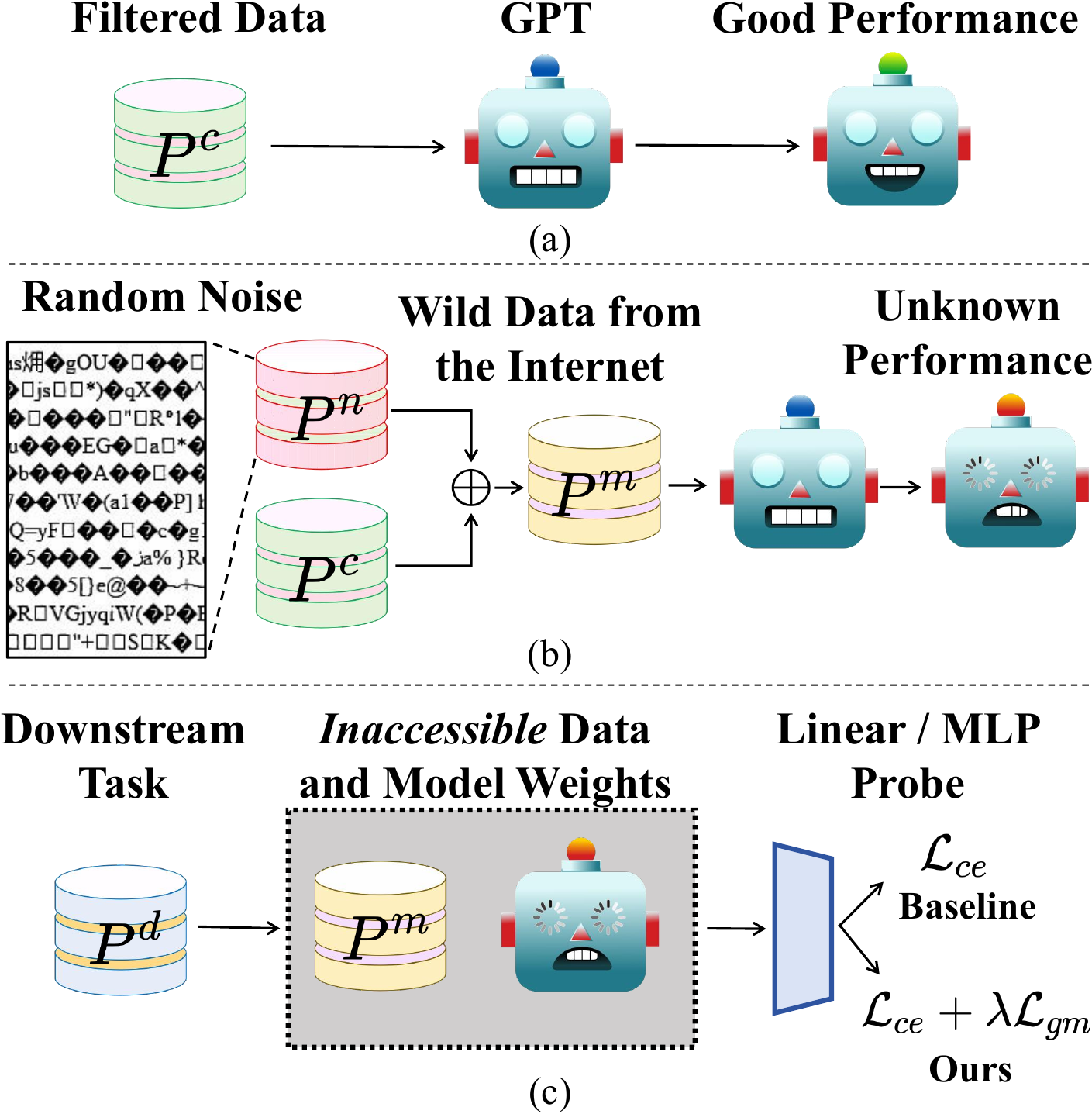} 
  \vspace{-4mm}
  \caption{Overview of the study and methodology. (a) The common scenario in which a GPT model, pre-trained on filtered data $P^c$, demonstrates robust performance. (b) When the pre-training dataset is contaminated with random noise $P^n$, the resultant language model may exhibit unpredictable behavior. (c) Our approach follows \citep{nml} and focuses on the effective fine-tuning of black-box noisy models for downstream tasks $P^d$.}
  \label{fig:1}
  \vspace{-5mm}
\end{wrapfigure}

On the other hand, further experiments reveal that a noisy model that exhibits a lower NTP loss experiences a 1.5\% decrease in accuracy on downstream tasks. This indicates that performance on downstream tasks is not solely rely upon the NTP loss. Given the common practice of fine-tuning a pre-trained foundation model rather than undergoing a full pre-training process from scratch, we opt to follow the work of \citep{nml} by exploring how to efficiently fine-tune language models only using extracted features for downstream tasks when the pre-training data and model weights are not accessible, which reflects real-world application scenarios for LLMs. To mitigate the potential adverse effects of noise, we propose a novel plug-and-play Local Gradient Matching (LGM) loss. This method involves artificially adding noise to the output features and minimizing the gradient difference between the noisy and original features. We also provide a theoretical analysis of the LGM loss. Interestingly, when applying the LGM loss to fine-tune clean models such as Llama-3 or ViT-L on 8 language and 14 vision datasets, we observe an unexpected improvement in accuracy, which effectively demonstrates the versatility and broad applicability of the LGM loss beyond its original intent of addressing noise-related issues. 


The remaining part is arranged as follows. In Section \ref{sec:2}, we summarize related works. In Section \ref{sec:3} and \ref{sec:4}, we follow a logical structure based on \textbf{``What-Why-How''}:
\vspace{-2mm}
\begin{itemize} 
   \item What: What is the effect of random noise? In Section \ref{sec:3}, we demonstrate through experiments how random noise affects NTP loss.

   \item Why: Why does it have this effect? The section further explores the underlying reasons by providing a rigorous theoretical analysis.

   \item How: How do we mitigate the potentially harmful effect on downstream tasks? In Section \ref{sec:4}, we introduce the LGM loss and provide a theoretical explanation.
\vspace{-2mm}
\end{itemize}
Section \ref{sec:4} further provides evidence of the novelty and effectiveness of LGM loss by conducting extensive experiments on 22 downstream tasks.

Our contributions are as follows: (1) We investigate the underexplored problem of random noise in pre-training datasets for language models. (2) We pre-train multiple GPT-2 models with parameter size up to 2.7B and the empirical results show that its influence on NTP loss is insignificant. Then we provide a theoretical analysis that can be extended to other domains. (3) We propose a novel blackbox fine-tuning LGM loss for downstream tasks, supported by comprehensive experimental and theoretical analysis. \textbf{Code, data, and model checkpoint weights are available at} \href{https://anonymous.4open.science/r/lmn-acl-E9D3}{this repo}.
\nopagebreak

\section{Related Works}
\label{sec:2}

\textbf{Pre-training Data Analysis for Language Model Training.} Elazar et.al \citep{wimbd} analyzed open-source datasets like The Pile and C4, uncovering significant amounts of low-quality content in these datasets. Exsiting works \citep{allenphysics, colmmc} highlighted the negative impact of such data on training. Despite these remarkable contributions, there remains a lack of understanding regarding the specific effects of random noise on language model performance. This paper aims to address this gap.

\textbf{Noisy Model Learning.} Our work draws significant inspiration from Noisy Model Learning (NML) \citep{nml}. In NML, the authors introduce noise into large datasets like ImageNet by randomly altering labels, then pre-train neural networks on these noisy datasets. The study reveals that moderate label noise enhances in-distribution (ID) sample classification, while out-of-distribution (OOD) performance deteriorates with increasing noise. This paper extends the concept of NML, presenting new theoretical insights and methodologies.

Due to space limitations, detailed related works are provided in Appendix~\ref{app:relatedwork}.
\vspace{-1.5mm}

\section{Revealing the Effect of Random Noise in Language Model Pre-training}
\label{sec:3}

In this section, we first pre-train multiple GPT-2 models on synthetic noisy OpenWebText corpus to investigate the impact of random noise in the pre-training data. We then provide a theoretical analysis of the results and validate our theory through experiments. Finally, we demonstrate that the insights gained from our investigation have broader applicability beyond the immediate scope of our study. In summary, Section~\ref{sec:3} focuses on understanding the impact of random noise by presenting detailed experiments and theoretical analyses. We delve into how random noise affects NTP loss and introduce insights into why these effects occur. It sets the foundation for our investigation.

The frequently used notation and their descriptions are shown in Appendix~\ref{sec:appa}.
\vspace{-1.5mm}
\subsection{Preliminary} Let $L$ denote the maximum context length of the language model and let $\mathcal{W}$ represent the model's vocabulary with size $V = |\mathcal{W}|$. We define $\mathcal{X}$ as the set of all discrete sentences that the model can represent, where $\mathcal{X} = \cup_{i=1}^{L}\{0, 1, \dots, V -1\}^{i} = \cup_{i=1}^{L}\mathcal{W}^{i}$ and $\{0, 1, \dots, V -1\}^{i}$ represents prefixes of length $i$. For any discrete set $A$, let $\Delta_A$ denote the set of all probability distributions defined on $A$. Given that next-token prediction (NTP) is actually a classification task given the prefix, we define joint probability distributions $P^c, P^n, P^m \in \Delta_{\mathcal{X} \times \mathcal{W}}$ where $P^c$ represents the distribution of clean data, $P^n$ represents the distribution of noise data, and $P^m$ represents the distribution of the mixed pre-training dataset. Since the noisy dataset can be viewed as the concatenation of clean data and random noise, it can be formalized by the Huber contamination model \citep{oodfz} as follows:
\vspace{-1mm}
\begin{equation}
  \label{eq:huber}
  P^m = \alpha P^n + (1 - \alpha) P^c
\end{equation}
\noindent where we use $\alpha$ to represent the noise proportion. An explanation of Equation~(\ref{eq:huber}) can be found in Appendix~\ref{app:huber}. For any joint probability distribution $P \in \Delta_{\mathcal{X} \times \mathcal{W}}$, let $P_X \in \Delta_{\mathcal{X}}$ and $P_{\cdot \vert X} \in \Delta_{\mathcal{W}}$ represent the marginal and conditional distribution of $P$ over $\mathcal{X}$ and $\mathcal{W}$. 

We use $\mathcal{H}$ to denote the hypothesis space (e.g., all possible parameters given the transformer architecture ). Define $h: \mathcal{X} \rightarrow \mathbb{R}^V \in \mathcal{H}$ as the language model and $\boldsymbol{p}^{h}_{\cdot \vert x}(w)$ as the $w$-th component of the probability distribution induced by $h(x)$. The NTP loss can be expressed as follows:
\vspace{-1mm}
\begin{equation}
  \label{eq:ntp}
  \mathcal{L}_{ntp}(P, h) = \mathbb{E}_{x \sim P_X} \mathbb{E}_{w \sim P_{\cdot | x}} \big[-\log(\boldsymbol{p}^h_{\cdot | x}(w))\big].
\end{equation}
\vspace{-5.5mm}

\subsection{Random Noise Data Generation}
\subsubsection{What is “Random Noise”?}
\label{sec:def}
Our study focuses on ``random noise'', informally defined as \textit{any token sequence that appears nonsensical with respect to the distributions humans can read}. In real-world text data crawled from the Internet, we encounter at least two types of random noise:
(1) \textit{Pure random data with a uniform distribution:} such noise may arise from hardware faults on backend servers, encrypted content, or website defense mechanisms. For example, Son et al. \citep{defweb} observe that “several defense techniques inject random noise during website rendering to degrade the attack success rate.”
\urldef{\oracleurl}\url{https://docs.oracle.com/cd//E19528-01//820-0890//WebSvr.html#wp35099}
(2) \textit{Decoding-error data with a distorted distribution:} this noise can originate from decoding errors, the vast volume of unmaintained websites \footnote{\scriptsize \url{https://community.cloudflare.com//t//website-showing-garbage-text}}, and even appears in Oracle’s official documentation \footnote{\scriptsize \oracleurl}. Although such data follows some distribution, it diverges drastically from clean data, and experiments show that a GPT-2 model trained on clean data exhibits significantly higher predictive entropy with these sequences.

Since neither of these noise types occurs in a normal dataset, in mathematical terms this implies that for any prefix $r$ sampled from $P^n_X$, the probability under the clean distribution $P^c_X(r)$ is zero. Thus, we make the following mild assumption as the formal definition of random noise:

\begin{assumption}
\label{assumption1}
$P^c$ and $P^n$ have disjoint support sets, i.e., $\mathrm{supp}(P^c)\cap \mathrm{supp}(P^n)=\emptyset$.
\end{assumption}
\vspace{-2mm}
Previous studies have focused on synthetic or low-quality data, which is similar to clean data distributions but contain incorrect or substandard information. An imperfect analogy is to treat next-token prediction as an image classification task: low-quality data corresponds to label noise in image classification, whereas random noise represents a fundamental corruption of the image distribution itself. This distinction marks the clear novelty of our work relative to prior efforts.

\vspace{-2mm}
\subsubsection{Data Generation Process}  
\vspace{-2mm}

Due to constrained computational resources, we are unable to collect a sufficiently large corpus of genuine random noise from live web crawls. Instead, we adopt a synthetic approach that faithfully emulates the two primary noise modalities mentioned above. We use OpenWebText \citep{openwebtext} as the clean dataset $D_c$ comprising $|D_c| \approx 8$ billion tokens as an alternative to the original WebText dataset.

\textbf{Uniform Random Noise.}
To mimic pure random corruption (e.g., hardware faults, encryption artifacts, or anti-crawling defenses \citep{defweb}), we generate sequences of token IDs following uniform distribution as the random noise data $D_n$. Concretely, let \(V=50256\) denote the tokenizer vocabulary size and \(\alpha\) the desired noise proportion. We sample
\begin{equation}
    \label{eq:dn}
    x_i \sim \mathrm{Uniform}\bigl(0, V-1\bigr)\,,\quad i=1,\dots, |D_n|,
\end{equation}
where \(|D_n|\) is chosen such that $\alpha=\frac{|D_n|}{|D_n|+|D_c|}$. Our experiments are primarily based on $D_n$. In the following text, unless otherwise specified, the term "random noise" refers to this particular type.

\textbf{Gaussian-Distributed Noise.}  
To capture decoding errors and web-rendering artifacts that follow certain skewed distribution, we further sample Gaussian noise data $D_g$ from a truncated discrete Gaussian distribution to mimic a learnable noise. Let \(\mu = (V-1)/2\) and \(\sigma\) be a chosen standard deviation (e.g., \(\sigma=500\)). For \(i=1,\dots, |D_g|\), we draw
\begin{equation}
    \label{eq:dg}
    x_i \sim \mathrm{DiscreteGaussian}(\mu, \sigma)\quad\text{clipped to }[0, V-1].
\end{equation}
\begin{wrapfigure}{r}{0.6\linewidth}
\vspace{-6mm}
\begin{algorithm}[H]
\caption{Synthetic Noise Generation}
\label{alg1}
\KwIn{Clean dataset \(D_C\), noise ratio \(\alpha\), maximum context length $L$, batch size $B$} \
\tcp{Data generation process} \
$D_n=\emptyset$ \\
\While{$\frac{|D_n|}{|D_n|+|D_c|} \le \alpha$} {
    sample $x_i$ according to Equation~(\ref{eq:dn}) \tcp{or (\ref{eq:dg})} \
    $D_n$ = {\ttfamily concat}$(D_n, x_i)$
}
$D_m$={\ttfamily concat}$(D_c, D_n)$ \\
\tcp{Pretraining Stage}
\While{training not converged} {
    \tcp{sample a mini-batch}
    {\ttfamily
    ix = torch.randint($|D_m|$ - $L$, ($B$)) \\
    x = [$D_m$[i:i+$L$] for i in ix] \\
    y = [$D_m$[i+1:i+1+$L$] for i in ix] \\
    }
    {\ttfamily l = cross\_entropy(x, y)} \\
    {\ttfamily backward\_loss\_and\_update\_optimizer()} \\
}
\end{algorithm}
\vspace{-8.375mm}
\end{wrapfigure}
Upon acquiring the noisy data, we directly concatenate it to $D_c$. Since token sampling during training is completely random, the positional arrangement of noise data - whether pre-appended, appended, or interleaved with clean data - becomes statistically irrelevant in practice. The complete data generation and pre-training procedures are formally described in Algorithm~\ref{alg1}. We set $\alpha$ to be 1\%, 5\%, and 20\% for $D_n$, and 5\% for $D_g$.

While our pioneering experiments employ a relatively simple simulation strategy, we argut that it nevertheless captures the essential characteristics of real-world random noise . By adopting Gaussian or random noise 
as an initial fully parameterizable testbed, we ensure analytical tractability and reproducibility, and it can readily be extended to more realistic data in future work.

\subsection{Experiments}
\begin{figure*}[t]
  \centering
  \begin{subfigure}{0.3\linewidth}
    \includegraphics[height=3.3cm, width=\linewidth]{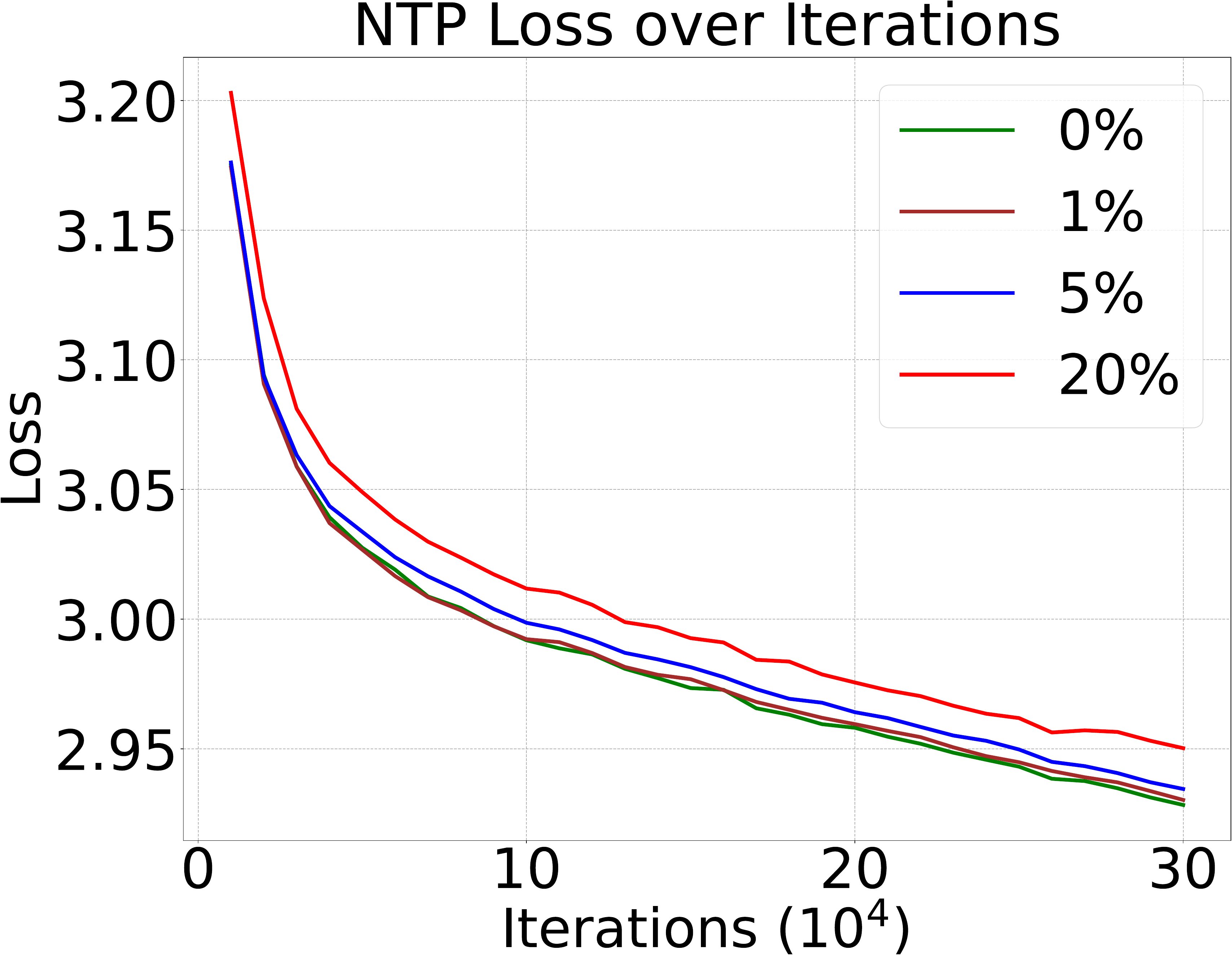}
    \caption*{(a)}
  \end{subfigure} \hfill
  \begin{subfigure}{0.3\linewidth}
    \includegraphics[height=3.3cm, width=\linewidth]{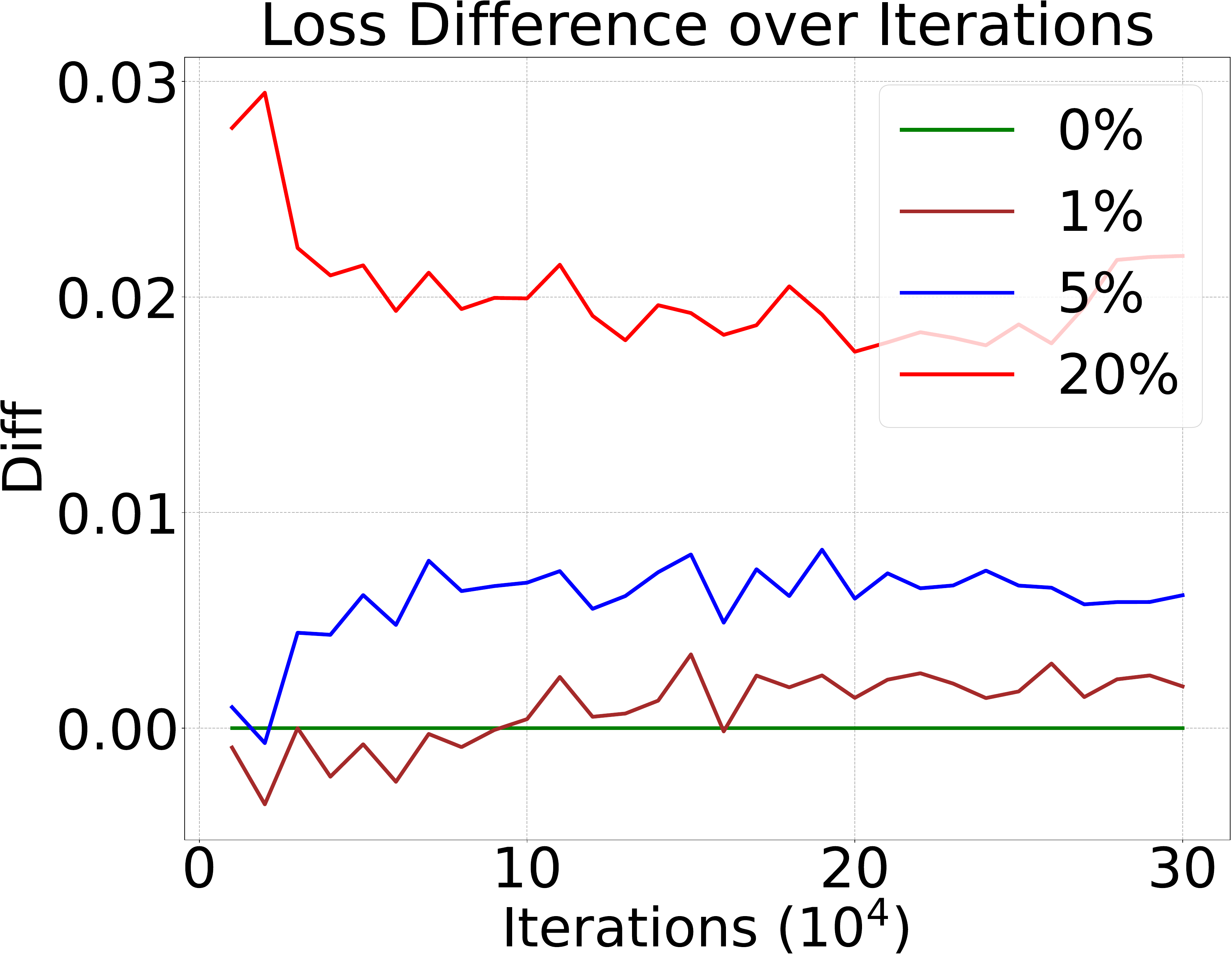}
    \caption*{(b)}
  \end{subfigure} \hfill
  \begin{subfigure}{0.35\linewidth}
    \includegraphics[height=3.3cm, width=\linewidth]{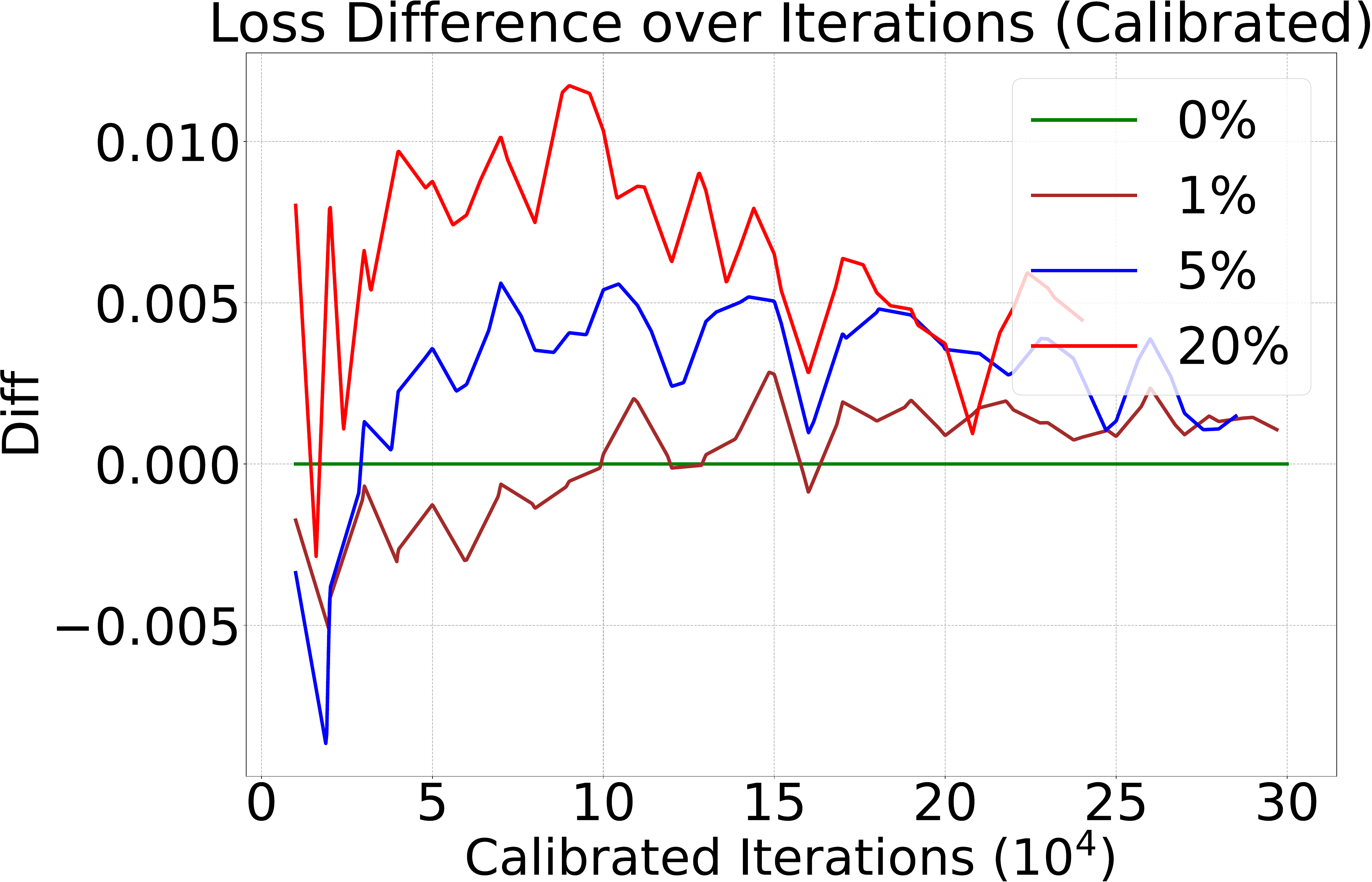}
    \caption*{(c)}
  \end{subfigure}
  \vspace*{-2.5mm}
  \caption{Next-token prediction loss on the clean OpenWebText validation set for GPT-2 124M models pre-trained on synthetic OpenWebText datasets with varying levels of random noise. (a) Trend of NTP loss as training proceeds. (b) Difference in NTP loss between the noisy and clean models after the same number of training iterations. (c) Difference in loss values after undergoing the same number of training iterations on clean OpenWebText data.}
  \label{fig:2}
  \vspace*{-4mm}
\end{figure*}
\begin{figure*}[t]
  \centering
  \begin{subfigure}{0.28\linewidth}
    \includegraphics[height=3.4cm, width=\linewidth]{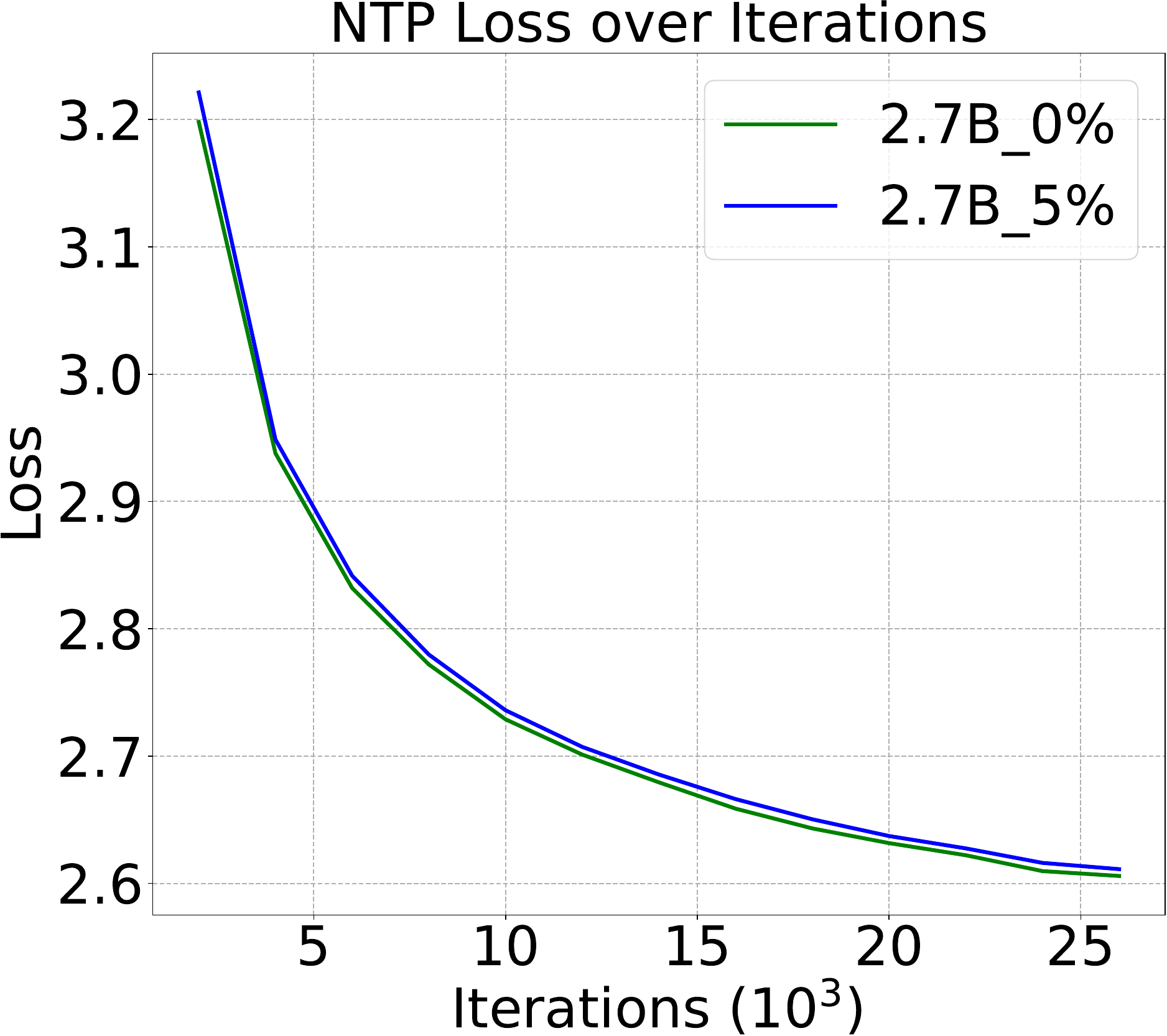}
    \caption*{(a)}
  \end{subfigure} \hfill
  \begin{subfigure}{0.30\linewidth}
    \includegraphics[height=3.4cm, width=\linewidth]{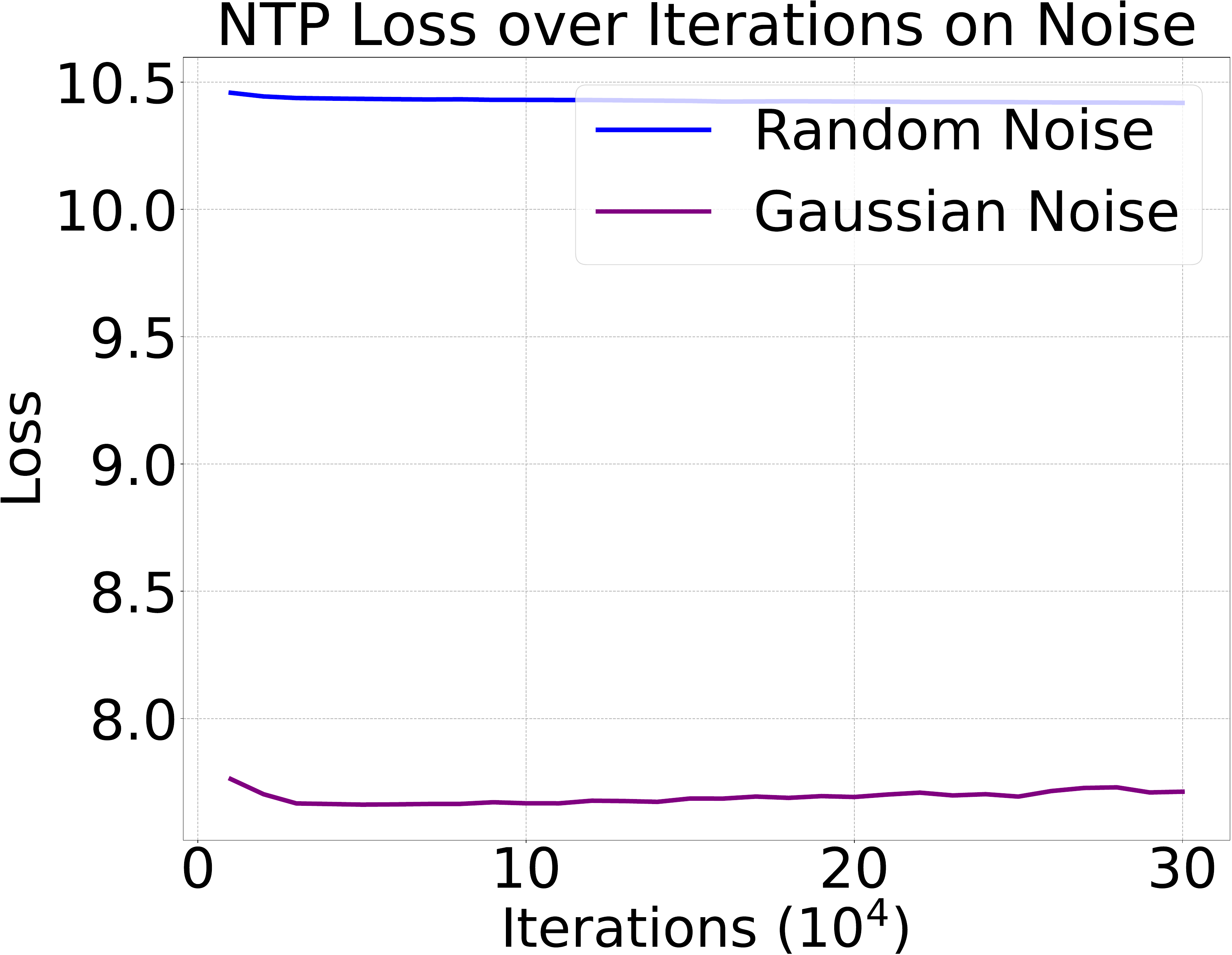}
    \caption*{(b)}
  \end{subfigure} \hfill
  \begin{subfigure}{0.32\linewidth}
    \includegraphics[height=3.4cm, width=\linewidth]{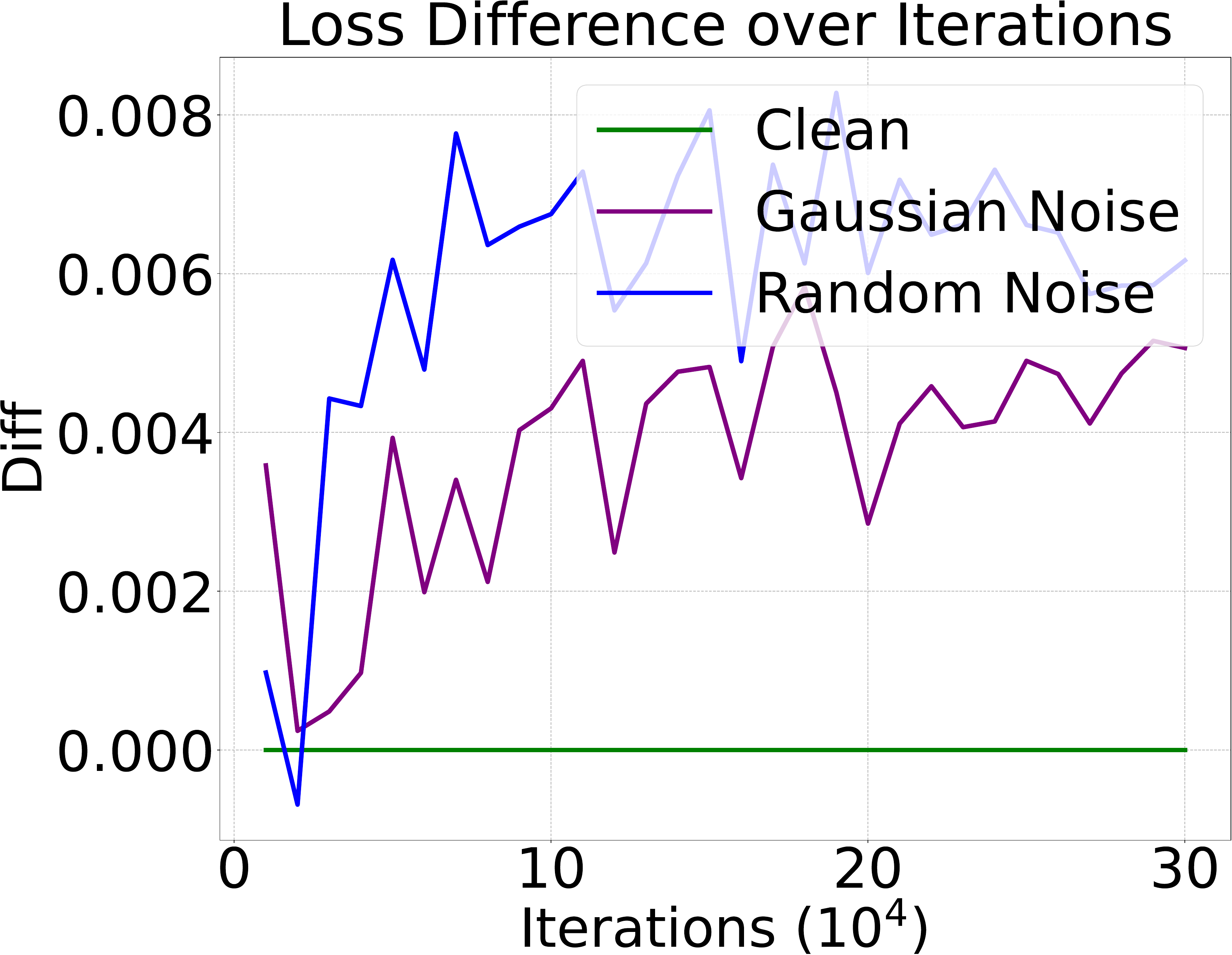}
    \caption*{(c)}
  \end{subfigure}
  \vspace*{-2.5mm}
  \caption{Validation experiments with the models trained with 5\% random noise.  (a) Loss trends on the OpenWebText validation set for GPT-2 2.7B model. (b) Comparison of the 124M training set loss between 5\% random noise and Gaussian noise. (c) Loss difference on the validation set for 124M models trained on datasets with 5\% random noise and 5\% Gaussian noise, respectively.}
  \label{fig:3}
  \vspace*{-7mm}
\end{figure*}

Our work is based on \href{https://github.com/karpathy/nanoGPT}{nanoGPT}. Specifically, we pretrain GPT-2 on clean or noisy dataset with two parameter configurations: 
The first adheres to the GPT-2 124M \citep{gpt2} architecture with 12 transformer layers, 12 attention heads, and 768-dimensional embeddings. The second follows the GPT-Neo 2.7B \citep{pile} specification comprising 32 layers, 20 attention heads, and 2560-dimensional embeddings. Constrained by computational resources, our primary experiments utilize the 124M variant, while a 2.7B model pre-trained on 5\% random noise data is strategically deployed to validate the scalability of our findings. 
After pretraining, the resulting model checkpoints are tested on the clean OpenWebText validation set, measuring the NTP loss for comparison. Further details regarding datasets and experimental parameters can be found in Appendix~\ref{sec:mer}.

In Figure~\ref{fig:2}, we illustrate the evolution of the NTP loss throughout the training process. Although random noise has a negative effect on the model's performance as expected, experimental results yield intriguing insights: 
the impact of random noise on the loss is disproportionately small. In fact, according to conventional statistical learning theory \citep{uml}, the performance of machine learning models degrades as the distribution shift between training and test data increases. For example, Ben-David et al. \citep{ben-david} have shown that the loss on the test set can be upper bounded by the sum of training set loss and total variation distance between the training and test distributions, which reaches the maximum value when the two distributions have disjoint support set.  \textbf{In other words, previous theory predicts that random noise will lead to worst-case model performance. However, our experiments show the contradictory results.}
For instance, 5\% of random noise only results in a 0.2\% increase in the NTP loss. This discrepancy becomes even smaller if the noisy models are calibrated to match the number of training iterations with the baselines trained on clean datasets.

As illustrated in Figure~\ref{fig:3}(a), the impact of random noise remains negligible even when the model size scales up to 2.7B parameters. Furthermore, when evaluating the trained model on Wikipedia and arXiv datasets for NTP loss (detailed in Appendix~\ref{arxiv}), we observe that random noise can even be beneficial in certain instances. These counter-intuitive experimental outcomes further corroborate the robustness of language models and provide insights into why pre-training on large-scale datasets that inevitably contain significant amounts of noise can still yield high-performing models. These somewhat unexpected findings naturally prompt us to explore the underlying reasons.

\subsection{Theoretical Analysis}
In the analysis below, we focus on the impact of random noise on NTP loss which is crucial for the performance on downstream tasks \citep{mathlm, whyhelp, sameloss, zcy2}. Specifically, we are interested in the difference of NTP Loss between a model $h^*$ trained on a noise-free dataset and a model $h$ trained with a noisy dataset. 



\begin{proposition}
\label{pro1}
Under Assumption~\ref{assumption1}, let $h^*$ be a model trained on $P^c$, with $\mathcal{L}_{ntp}(P^c,h^*)=-\log p_c$ and $\mathcal{L}_{ntp}(P^n,h^*)=-\log p_n$. When the model $h$ is trained on a mixed distribution $P^m$ which includes noise, it attempts to fit $P^n$, leading to an increase in the loss on the clean distribution $P^c$, such that $\mathcal{L}_{ntp}(P^c,h)=-\log(p_c-\epsilon)$ and $\mathcal{L}_{ntp}(P^n,h)=-\log(p_n+\epsilon/k)$ for some $\epsilon > 0$ ($k$ can be shown to be $\Omega(e^{\mathcal{L}_{ntp}(P^n,h)})$). Let $\eta = \alpha p_c - (1-\alpha)kp_n$.
We arrive at the following conclusions:

(1) If $\alpha \leq \frac{kp_n}{p_c + kp_n}$, then for any $0 < \epsilon <p_c$, we have $\mathcal{L}_{ntp}(P^m, h) \ge \mathcal{L}_{ntp}(P^m, h^*)$. This means that when $\alpha$ is sufficiently small, the global minimum on $P^m$ will not be affected by noise.

(2) If $\alpha > \frac{kp_n}{p_c + kp_n}$, then for $\epsilon < \eta$, it holds that $\mathcal{L}_{ntp}(P^m,h) < \mathcal{L}_{ntp}(P^m,h^*)$. This suggests that if $\alpha$ is large enough, the impact on the optimal hypothesis is at least as much as $\alpha p_c - (1-\alpha)kp_n$.

(3) When $\alpha < \frac{1}{3}$ and $k>\frac{\alpha(1-3\alpha)p_c}{(1-\alpha)(2-3\alpha)p_n}$, for $\epsilon \ge 3\eta$ we get $\mathcal{L}_{ntp}(P^m,h^*) < \mathcal{L}_{ntp}(P^m,h)$. Similarly, it can be shown that $\epsilon$ does not exceed $2\eta$ when $\alpha > \max\left(\frac{kp_n}{p_c + kp_n}, \frac{1}{2}\right)$ and $k > \frac{(2\alpha-1)p_c}{2(1-\alpha)p_n}$. This indicates that when $k$ is sufficiently large, the effect of noise is at most $\mathcal{O}(\alpha p_c - (1-\alpha)kp_n)$.

\end{proposition}

The proof can be found in Appendix~\ref{proof_pro1}. 
Proposition~\ref{pro1} primarily investigates the performance gap between models trained on $P^m$ and those on $P^c$. It is proved that when $\alpha$ is small enough, the presence of noise has no impact on the optimal model on $P^m$. Even as $\alpha$ approaches $\frac{1}{3}$ or even $\frac{1}{2}$, as long as $k$ is large enough (the analysis regarding $k$ and other parameters is detailed in Appendix~\ref{app:explain}), the loss induced by noise, $\epsilon$, does not exceed $\mathcal{O}(\alpha p_c - (1-\alpha)kp_n)$. Given that $k$ is much greater than 1, this implies $\epsilon$ is much smaller than $\alpha p_c$. This explains the observed experimental results.

With these theoretical results in hand, we then conduct multiple experiments to substantiate their validity. First, we plot the trend of NTP loss on random and Gaussian noise within the \textit{training set} throughout the learning process, as shown in Figure~\ref{fig:3}(b). It is evident that the loss decreases at a very slow rate, indicating that the model struggles to efficiently learn the distribution of noise. 
Furthermore, the loss on Gaussian noise is lower than that on the random noise. According to Proposition~\ref{pro1}, since the Gaussian distribution corresponds to a high $p_n$, we can \textbf{predict} that a model trained on Gaussian noise will exhibit a lower loss on $P^c$. Figure~\ref{fig:3}(c) confirms our prediction, thus validating the proportions.

\subsection{Generalized Impact beyond Random Noise}
In addition to providing explanations regarding the impact of random noise on pre-training language models, we aim to extend our proposed theory to other areas, therefore demonstrating the practical value of our research findings.

One immediate direction is the training of multilingual models \citep{mlbert1, mlbert2, yb}. Apparently, tokens corresponding to different languages are usually distinct, and the sentences in one language usually do not appear in another language. Therefore, the distributions of different languages naturally satisfy Assumption~\ref{assumption1}. For example, in an English-Chinese bilingual model, let $P^c$ represent English and $P^n$ represent Chinese. Supposing the pretraining corpus consists of an equal distribution of English and Chinese, and given that the two distributions are natural languages, we can assume that $p_c \approx p_n$, leading to $\epsilon \approx 0$. This provides a theoretical foundation for the success of multilingual models. See Appendix~\ref{app:multilin} for more details.


Beyond the language modality, we demonstrate that our theoretical framework naturally extends to emerging native multimodal autoregressive large models (MLLMs). For example, VARGPT \citep{vargpt,vargpt1} uses a visual tokenizer for image understanding and generation, while UniAudio \citep{uniaudio} achieves state-of-the-art audio generation by modeling the joint distribution of audio and text tokens. Taking audio foundation models as a representative case, they typically employ a neural codec \citep{almtoken} to discretize continuous audio signals into token sequences. The multimodal input is then formatted as a concatenated sequence with text tokens before audio tokens, enabling standard next-token prediction during pretraining to learn conditional audio generation from textual prompts.

This paradigm inherently requires simultaneous modeling of text and audio token distributions. Apparently, the sets of audio tokens and text tokens are disjoint, thereby satisfying Assumption~\ref{assumption1}. Therefore, Proposition~\ref{pro1} provides a theoretical justification for the empirical observation that joint distribution modeling does not compromise model performance. For example,UniAudio has reported that although they mask text tokens during NTP pretraining to focus solely on audio prediction, empirical results show no significant improvement in loss metrics compared to the full joint prediction approach. This alignment between theoretical predictions and empirical observations underscores the practical relevance of our framework for analyzing modern multimodal architectures.
\vspace{-4mm}

\section{Reducing the Noise with Local Gradient Matching}
\label{sec:4}

\begin{table*}[t]
  \centering
  \caption{Accuracy on 4 text classification benchmark. 0\% represents a model trained on $P^c$, 1\% and so on denote the proportion of random noise. $^*$ cited from \citep{mathlm}.}
  \resizebox{\textwidth}{!}{
    \begin{tabular}{c|cc|cc|cc|cc|cc}
      \toprule
                                    & \multicolumn{2}{c|}{SST-2}                            & \multicolumn{2}{c|}{SST-fine}                         & \multicolumn{2}{c|}{20newsgroup}                             & \multicolumn{2}{c|}{CR}                               & \multicolumn{2}{c}{Avg}         \\
                                    & Linear                    & MLP                       & Linear                    & MLP                       & Linear                    & MLP                       & Linear                    & MLP                       & Linear         & MLP            \\ \hline
      OpenAI's GPT-2$^*$            & 87.4                      & /                         & 49.2                      & /                         & 63.7                      & /                         & 86.8                      & /                         & 71.75          & /              \\ \hline
      0\%                           & 86.71 $\pm$ 0.85          & 87.36 $\pm$ 0.33          & 49.19 $\pm$ 0.32          & 49.18 $\pm$ 0.02          & 63.12 $\pm$ 0.37          & 62.70 $\pm$ 0.86           & 85.65 $\pm$ 0.88          & 84.86 $\pm$ 0.36          & 71.16          & 71.02          \\
      0\% + $\mathcal{L}_{gm}$      & \textbf{87.42 $\pm$ 0.73} & \textbf{87.86 $\pm$ 0.04} & \textbf{49.72 $\pm$ 0.27} & \textbf{49.81 $\pm$ 0.97} & \textbf{63.69 $\pm$ 0.59} & \textbf{62.95 $\pm$ 0.13} & \textbf{86.58 $\pm$ 0.22} & \textbf{86.45 $\pm$ 0.73} & \textbf{71.85} & \textbf{71.76} \\ \hline
      1\%                           & 87.25 $\pm$ 0.79          & \textbf{87.53 $\pm$ 0.27} & 49.32 $\pm$ 0.72          & 49.45 $\pm$ 0.56          & 63.71 $\pm$ 0.02          & 64.65 $\pm$ 0.06          & 84.86 $\pm$ 0.98          & 84.59 $\pm$ 0.59          & 71.28          & 71.55          \\
      1\% + $\mathcal{L}_{gm}$      & \textbf{87.64 $\pm$ 0.91} & 87.25 $\pm$ 0.44          & \textbf{49.59 $\pm$ 0.73} & \textbf{50.01 $\pm$ 0.05} & \textbf{63.92 $\pm$ 0.65} & \textbf{64.72 $\pm$ 0.76} & \textbf{85.12 $\pm$ 0.07} & \textbf{85.25 $\pm$ 0.29} & \textbf{71.56} & \textbf{71.80} \\ \hline
      5\%                           & 86.92 $\pm$ 0.98          & 87.23 $\pm$ 0.41           & 49.04 $\pm$ 0.11          & \textbf{50.09 $\pm$ 0.53} & 63.27 $\pm$ 0.79          & 62.09 $\pm$ 0.28          & 85.30 $\pm$ 0.63          & \textbf{84.32 $\pm$ 0.78} & 71.13          & \textbf{70.93} \\
      5\% + $\mathcal{L}_{gm}$      & \textbf{87.19 $\pm$ 1.02} & \textbf{87.61 $\pm$ 0.51} & \textbf{49.82 $\pm$ 0.17} & 48.95 $\pm$ 0.89          & \textbf{63.78 $\pm$ 0.93} & \textbf{62.37 $\pm$ 0.56} & \textbf{85.57 $\pm$ 0.43} & 84.19 $\pm$ 0.69          & \textbf{71.59} & 70.78         \\ \hline
      20\%                          & 86.60 $\pm$ 1.28           & 86.60 $\pm$ 0.81           & 49.45 $\pm$ 0.78          & 49.63 $\pm$ 0.01          & 63.47 $\pm$ 0.64          & 64.16 $\pm$ 0.92          & \textbf{85.32 $\pm$ 0.60} & \textbf{85.45 $\pm$ 0.86} & 71.26          & 71.26          \\
      20\% + $\mathcal{L}_{gm}$     & \textbf{87.2 $\pm$ 0.99}  & \textbf{86.87 $\pm$ 0.78} & \textbf{49.68 $\pm$ 0.55} & \textbf{50.40 $\pm$ 0.46}  & \textbf{63.58 $\pm$ 0.08} & \textbf{64.21 $\pm$ 0.78} & 85.25 $\pm$ 0.90          & 85.52 $\pm$ 0.24          & \textbf{71.42} & \textbf{71.75} \\ \hline
      Gaussian                      & 85.22 $\pm$ 0.24          & 86.82 $\pm$ 0.72          & 46.15 $\pm$ 0.51          & 49.59 $\pm$ 0.76          & 63.72 $\pm$ 0.35          & \textbf{64.40 $\pm$ 0.76}  & 84.06 $\pm$ 0.74          & \textbf{83.53 $\pm$ 0.70} & 69.78          & 71.08          \\
      Gaussian + $\mathcal{L}_{gm}$ & \textbf{85.94 $\pm$ 0.55} & \textbf{87.25 $\pm$ 0.36} & \textbf{48.23 $\pm$ 0.69} & \textbf{50.29 $\pm$ 0.70} & \textbf{64.06 $\pm$ 0.73} & 64.29 $\pm$ 0.94          & \textbf{84.46 $\pm$ 0.33} & 83.29 $\pm$ 0.47          & \textbf{70.67}          & \textbf{71.45}          \\
    \bottomrule
    \end{tabular}}
  \vspace{-2mm}
  \label{tab:1}
  \vspace{-2mm}
\end{table*}

\begin{table*}[t]
  \centering
  \caption{Accuracy of LLMs on 4 natural language understanding benchmark.}
  \resizebox{\textwidth}{!}{
    \begin{tabular}{c|cc|cc|cc|cc|cc}
      \toprule
                                                 & \multicolumn{2}{c|}{BBC}                              & \multicolumn{2}{c|}{Balanced COPA}                    & \multicolumn{2}{c|}{MRPC}                             & \multicolumn{2}{c|}{WiC}                              & \multicolumn{2}{c}{Avg}         \\
                                                 & Linear                    & MLP                       & Linear                    & MLP                       & Linear                    & MLP                       & Linear                    & MLP                       & Linear         & MLP            \\ \hline
      Llama-3-8B                                 & 96.90 $\pm$ 0.40          & 97.50 $\pm$ 0.20          & 69.00 $\pm$ 0.20          & \textbf{65.60 $\pm$ 0.50} & 72.00 $\pm$ 0.81          & 67.53 $\pm$ 0.93          & 64.14 $\pm$ 0.56          & 59.07 $\pm$ 0.34          & 75.51          & 72.42          \\
      Llama-3-8B + $\mathcal{L}_{gm}$            & \textbf{98.00 $\pm$ 0.50} & \textbf{98.20 $\pm$ 0.40} & \textbf{70.80 $\pm$ 1.70} & 64.80 $\pm$ 0.20          & \textbf{74.89 $\pm$ 0.40} & \textbf{74.14 $\pm$ 1.49} & \textbf{64.71 $\pm$ 0.94} & \textbf{64.21 $\pm$ 0.83} & \textbf{77.10} & \textbf{75.33} \\ \hline
      Llama-3-8B-Instruct                        & 96.80 $\pm$ 0.70          & 96.90 $\pm$ 0.30          & 87.80 $\pm$ 0.70          & 88.80 $\pm$ 0.60          & 72.57 $\pm$ 0.26          & 71.42 $\pm$ 0.13          & 65.92 $\pm$ 0.53          & 61.85 $\pm$ 0.59          & 80.77          & 79.74          \\
      Llama-3-8B-Instruct + $\mathcal{L}_{gm}$   & \textbf{97.70 $\pm$ 0.20} & \textbf{97.80 $\pm$ 0.40} & \textbf{88.40 $\pm$ 0.90} & \textbf{89.60 $\pm$ 0.50} & \textbf{77.79 $\pm$ 0.58} & \textbf{76.81 $\pm$ 0.20} & \textbf{68.64 $\pm$ 0.26} & \textbf{67.71 $\pm$ 0.51} & \textbf{83.13} & \textbf{82.98} \\ \hline
      Llama-3.2-3B-Instruct                      & 97.30 $\pm$ 0.60          & 97.20 $\pm$ 0.80          & 80.40 $\pm$ 0.90          & \textbf{79.60 $\pm$ 0.20} & 77.79 $\pm$ 0.52          & 72.57 $\pm$ 0.31          & 64.07 $\pm$ 0.82          & 57.50 $\pm$ 0.35          & 79.89          & 76.71          \\
      Llama-3.2-3B-Instruct + $\mathcal{L}_{gm}$ & \textbf{97.60 $\pm$ 0.10} & \textbf{97.80 $\pm$ 0.30} & \textbf{81.60 $\pm$ 1.00} & 79.40 $\pm$ 0.10          & \textbf{78.43 $\pm$ 0.78} & \textbf{76.57 $\pm$ 1.12} & \textbf{64.35 $\pm$ 0.62} & \textbf{62.64 $\pm$ 0.07} & \textbf{80.49} & \textbf{79.10} \\ \hline
      Qwen2.5-1.5B-Instruct                      & 97.00 $\pm$ 0.30          & 96.60 $\pm$ 0.70          & 80.80 $\pm$ 0.70          & 82.20 $\pm$ 0.50           & 74.49 $\pm$ 0.71          & 73.39 $\pm$ 0.90          & 65.92 $\pm$ 0.45          & 61.64 $\pm$ 0.20          & 79.55          & 78.45          \\
      Qwen2.5-1.5B-Instruct + $\mathcal{L}_{gm}$ & \textbf{97.40 $\pm$ 0.10} & \textbf{97.20 $\pm$ 0.80} & \textbf{84.00 $\pm$ 0.90} & \textbf{83.40 $\pm$ 0.30}  & \textbf{79.65 $\pm$ 0.62} & \textbf{78.37 $\pm$ 0.84} & \textbf{67.71 $\pm$ 0.49} & \textbf{66.92 $\pm$ 0.55} & \textbf{82.19} & \textbf{81.47} \\ \hline
      Qwen2.5-7B-Instruct                        & 96.30 $\pm$ 0.30          & 96.70 $\pm$ 0.50          & 94.60 $\pm$ 0.90          & 95.80 $\pm$ 0.40          & 83.71 $\pm$ 0.92          & 76.81 $\pm$ 0.51          & 68.92 $\pm$ 0.41          & 64.92 $\pm$ 0.18          & 85.88          & 83.55          \\
      Qwen2.5-7B-Instruct + $\mathcal{L}_{gm}$   & \textbf{97.10 $\pm$ 0.80} & \textbf{97.40 $\pm$ 0.20} & \textbf{95.60 $\pm$ 0.50} & \textbf{96.00 $\pm$ 0.80} & \textbf{84.98 $\pm$ 0.12} & \textbf{83.13 $\pm$ 0.49} & \textbf{72.28 $\pm$ 0.98} & \textbf{70.14 $\pm$ 0.94} & \textbf{87.49} & \textbf{86.66} \\ \bottomrule
      \end{tabular}}
    \vspace{-2mm}
    \label{tab:2}
    \vspace{-4mm}
\end{table*}

In Section~\ref{sec:3}, we know that the influence of noise on NTP loss is rather small. However, Figure~\ref{fig:3}(c) and Table~\ref{tab:1} show that the Gaussian noise-trained model with lower NTP loss suffers a 1.5\% decrease in accuracy in downstream tasks. Therefore, in order to tame the potential influence in downstream tasks, we propose a novel black-box fine-tuning method termed Local Gradient Matching loss. Extensive experiments across 8 natural language understanding and 14 image classification benchmark datasets further demonstrate that the proposed method consistently enhances performance across different backbones and modalities. We also provide a theoretical analysis.

It is important to emphasize that the rationale for investigating pre-training in Section~\ref{sec:3} while proposing a fine-tuning-based approach for downstream tasks in this section stems from current conventions: the prevailing practice is adapting off-the-shelf foundation models to downstream tasks through data-specific fine-tuning, rather than undertaking pre-training from scratch. Substantial random noise primarily emerges in industrial-scale datasets, but pre-training on such massive datasets exceeds our computational resources. Furthermore, our choice to fine-tune the downstream task head strictly follows \citep{nml}, aiming to enhance model performance under the assumption of black-box model constraints. Additionally, it should be clarified that our core experimental validation resides in Table~\ref{tab:1}, while the additional experiments conducted on both LLMs and visual models serve as supplementary evidence to demonstrate the generalization capability of our proposed method - this constitutes an extended validation rather than the primary focus of this paper.

\subsection{Method}
In the preceding analysis, we demonstrate that the population-level loss function is only marginally affected by random noise. However, during the SGD training process, its presence introduces certain noise into the gradients. 
Prior studies \citep{gradnoise1, gradnoise2} have shown that artificially added gradient noise can hurt the model's generalization. Therefore, we propose explicitly enhancing the denoising capabilities of the downstream task head by aligning local gradients.

\begin{table*}[t]
  \centering
  \caption{Average accuracy of 5 vision backbone models on \textbf{14} commonly-used vision datasets.}
  \resizebox{\textwidth}{!}{
  \begin{tabular}{c|cc|cc|cc|cc|cc}
  \toprule
  Model                  & \multicolumn{2}{c|}{EfficientNet-B3} & \multicolumn{2}{c|}{ResNetv2-152x2} & \multicolumn{2}{c|}{Swin-L}       & \multicolumn{2}{c|}{ConvNext-L} & \multicolumn{2}{c}{ViT-L}       \\
  Pre-training Data      & \multicolumn{2}{c|}{JFT-300M}        & \multicolumn{2}{c|}{ImageNet-21K}   & \multicolumn{2}{c|}{ImageNet-21K} & \multicolumn{2}{c|}{Laion-2B}   & \multicolumn{2}{c}{Laion-2B}    \\
  Fine-tuning Method     & Linear            & MLP              & Linear           & MLP              & Linear          & MLP             & Linear         & MLP            & Linear         & MLP            \\ \hline
  w/o $\mathcal{L}_{gm}$ & 73.27             & \textbf{76.62}            & 78.14            & 79.60            & 81.43           & 84.19           & 82.89          & 85.71          & 86.86          & 89.12          \\ \hline
  w/ $\mathcal{L}_{gm}$  & \textbf{74.02}    & 75.90   & \textbf{79.49}   & \textbf{79.94}   & \textbf{82.70}  & \textbf{84.42}  & \textbf{84.07} & \textbf{86.27} & \textbf{88.03} & \textbf{89.31} \\ \bottomrule
  \end{tabular}}
  \vspace{-2mm}
  \label{tab3}
  \vspace{-3mm}
\end{table*}

Specifically, let $C$ denote the number of classes in the downstream task, and let $g_{\theta}: \mathbb{R}^d \rightarrow \mathbb{R}^C$ represent the linear or MLP classification head parameterized by $\theta$. Let $t^*$ be the feature extracted by $h^*$, $t$ be the feature extracted by $h$, and $y$ be the corresponding label. 
\begin{wrapfigure}{r}{0.5\columnwidth} 
  \centering
  \includegraphics[width=\linewidth]{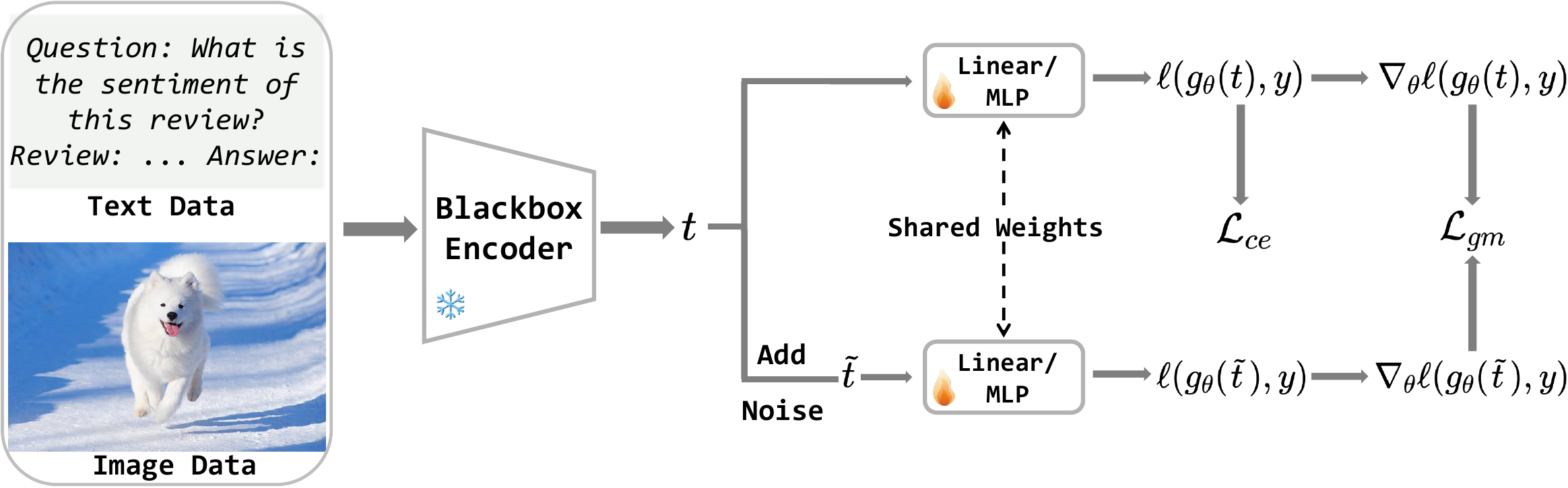}
  \vspace*{-6mm}
  \caption{Overview of the proposed Local Gradient Mathcing scheme.}
  \label{fig:5}
  \vspace*{-2mm}
\end{wrapfigure}
Let $\ell(\hat{y}, y )$ be the loss function(typically cross-entropy), and $\mathcal{L}_{ce}(\mathcal{D}, g_{\theta}) = \mathbb{E}_{(t, y) \sim \mathcal{D}} \ell(g_{\theta}(t), y)$ be the population-level loss where $\mathcal{D}$ represents the joint distribution of downstream features and labels. Due to the additional randomness introduced by $h$ as a result of noise, $t$ can be viewed as $t^*$ perturbed by minor disturbances. If both $t^*$ and $t$ were known, their distribution could be aligned to achieve denoising. However, in practical applications, it is challenging to obtain $t^*$. To construct contrastive sample pairs without $t^*$, we add Gaussian noise to $t$ to obtain $\hat{t}$:
\begin{table}[t]
  \centering
  \begin{minipage}[t]{0.44\textwidth} 
    \centering
    \vspace{0pt} 
    \caption{Evaluation of our method combined with SOTA fine-tuning techniques utilizing BERT-Large ($^*$ cited from \citep{noisebert2}).}
    \resizebox{\linewidth}{!}{
      \begin{tabular}{c|c|c|c|c}
      \toprule
      & RTE            & MRPC           & CoLA           & STS-B          \\ \hline
      L$^2$-SP$^*$                  & 70.58          & \textbf{87.74} & 60.54          & 89.38          \\
      L$^2$-SP + $\mathcal{L}_{gm}$ & \textbf{71.25} & 87.62          & \textbf{61.79} & \textbf{89.62} \\ \hline
      SMART$^*$                     & 72.23          & 87.86          & 63.16          & 90.11          \\
      SMART + $\mathcal{L}_{gm}$    & \textbf{72.94} & \textbf{88.61} & \textbf{63.28} & \textbf{90.42} \\ \hline
      LNSR$^*$                      & 73.31          & 88.50          & 63.35          & 90.23          \\
      LNSR + $\mathcal{L}_{gm}$    & \textbf{73.95} & \textbf{89.42} & \textbf{63.82} & \textbf{90.47} \\ \bottomrule
      \end{tabular}}
    \label{tab4}
  \end{minipage}%
  \hfill 
  \begin{minipage}[t]{0.545\textwidth}
    \centering
    \vspace{0pt} 
    \caption{Ablation Study. To investigate the effects of reducing $\mathcal{L}_{gm}$, we separately reduce the norm or increase the cosine similarity.}
    \resizebox{\linewidth}{!}{
    \begin{tabular}{c|cc}
      \toprule
      \multirow{2}{*}{Method}                                                                                 & \multicolumn{2}{c}{SST-2}       \\
                                                                                                              & Linear         & MLP            \\ \hline
      0\%                                                                                                     & 86.71          & 87.36          \\ \hline
      0\% + $||\nabla_{\theta}  \ell(g_{\theta}(t), y )||_2$                                                  & 87.04          & 87.24          \\ \hline
      0\% + $\cos (\nabla_{\theta}  \ell(g_{\theta}(t), y ), \nabla_{\theta}  \ell(g_{\theta}(\hat{t}), y ))$ & 86.89          & 87.52          \\ \hline
      0\% + $\mathcal{L}_{gm}$                                                                                & \textbf{87.42} & \textbf{87.86} \\ \bottomrule
      \end{tabular}
    }
    \label{tab5}
  \end{minipage}
\end{table}
\begin{equation}
  \label{eq:perturb}
  \hat{t} = t + \gamma \cdot \delta
\end{equation}
where $\delta \sim \mathcal{N}(\mathbf{0}, \mathbf{I}_n)$ denotes the standard normal distribution noise. Our objective is to minimize the discrepancy between the distributions of $g_{\theta}(t)$ and $g_{\theta}(\hat{t})$. Instead of the conventional regularization term $||g_{\theta}(t) - g_{\theta}(\hat{t})||_2$, we propose to align the gradient difference:
\vspace{-1mm}
\begin{equation}
  \label{eq:lgm}
  \begin{aligned}
    \mathcal{L}_{gm}(\theta) &= || \mathbb{E}_{(t, y) \sim \mathcal{D}} \nabla_{\theta}  \ell(g_{\theta}(t), y ) - \mathbb{E}_{(\hat{t}, y) \sim \hat{\mathcal{D}}} \nabla_{\theta}  \ell(g_{\theta}(\hat{t}), y ) ||_2
  \end{aligned}
\end{equation}
Intuitively, if the gradients with respect to $t$ and $\hat{t}$ can be perfectly aligned, then the classification head is insensitive to small perturbations in the input, suggesting that it possesses some denoising capability. Consequently, it should be able to mitigate the noise in $t$, bringing it closer to $t^*$.

\subsection{Theoretical Analysis}
\label{sec:4.2}
To theoretically support the proposed method, we investigate the properties of Equation~(\ref{eq:lgm}) and find that it can be upper bounded by the smoothness, input flatness, and loss function value at $\theta$. Concretely, since we set $\gamma$ in Equation~(\ref{eq:perturb}) to be small, the perturbation can be considered to distribute within an open ball $B(0, \rho)$. Consequently, we have the following result:
\begin{proposition}
  \label{pro2}
  Suppose $\ell(g_{\theta}(t), y)$ is $\beta$-smooth with $\rho$-input flatness $R_{\rho}(\theta)$ ({\rm{c.f.}} Appendix~\ref{app:4.2}), for any $\theta \in \Theta$:
  \begin{equation}
  \mathcal{L}_{gm}(\theta) \le 2 \beta + 2 \mathcal{L}_{ce}(\mathcal{D}, g_{\theta}) + R_{\rho}(\theta).
  \end{equation}
\end{proposition}

Proposition~\ref{pro2} demonstrates that $\mathcal{L}_{gm}$ is closely associated with the smoothness of the loss function in both the parameter space and the input space. As a flat minima is widely acknowledged to benefit the generalization of neural networks \citep{flat1, flat2}, it explains the effectiveness of $\mathcal{L}_{gm}$. The final loss function is:
  \vspace{-2mm}
\begin{equation}
  \mathcal{L} = \mathcal{L}_{ce} + \lambda \mathcal{L}_{gm}
  \label{eq:final}
    \vspace{-2mm}
\end{equation}

\subsection{Experiments}
We first conduct extensive experiments using trained GPT-2 models. Then, to further validate the novelty and effectiveness of the LGM loss, we conduct additional experiments using Llama-3 \citep{llama3} and vision models. These experiments are intended to showcase the generalizability of our approach beyond the specific context of GPT-2, demonstrating its applicability across different types of models and tasks. While these experiments enrich our study, \textbf{they are not the core focus but rather supplementary evidence} supporting the broader applicability of our proposed solution. 
Details can be found in Appendix~\ref{app:exp4-detail}.

\begin{wraptable}{r}{4cm} 
  \centering
  \caption{Hyperparameter sensitivity experiments on DTD with ConvNext as the backbone.}
  \resizebox{0.9\linewidth}{!}{
  \begin{tabular}{cc|cc}
      \toprule
      \multirow{2}{*}{$\gamma$} & \multirow{2}{*}{$\lambda$} & \multicolumn{2}{c}{DTD} \\
                                &                            & Linear      & MLP       \\ \hline
      0.001                     & 0.05                      & 76.48       & 78.42     \\ \hline
      0.05                      & 0.1                       & 76.81       & 79.21     \\ \hline
      0.1                       & 0.15                        & 76.59       & 79.37     \\ \bottomrule
      \end{tabular}
    }
  \label{tab6}
\end{wraptable}

We validate the performance of $\mathcal{L}_{gm}$ on models pre-trained with noisy data using four commonly used classification datasets: SST-2, SST-fine, 20newsgroup, and CR. The training hyperparameters follow those of \citep{mathlm}, where $\gamma=0.01$ and $\lambda=0.15$ apply to all four experiments. In line with the approach described by \citep{nml}, we freeze the model parameters and only fine-tune a linear or MLP classifier head. As shown in Table~\ref{tab:1}, our model achieves competitive results without reaching the number of training iterations of GPT-2, and $\mathcal{L}_{gm}$ consistently boosts performance.

\begin{figure*}[t]
  \centering
  \begin{subfigure}[t]{0.25\linewidth}
    \includegraphics[height=4cm, width=4cm]{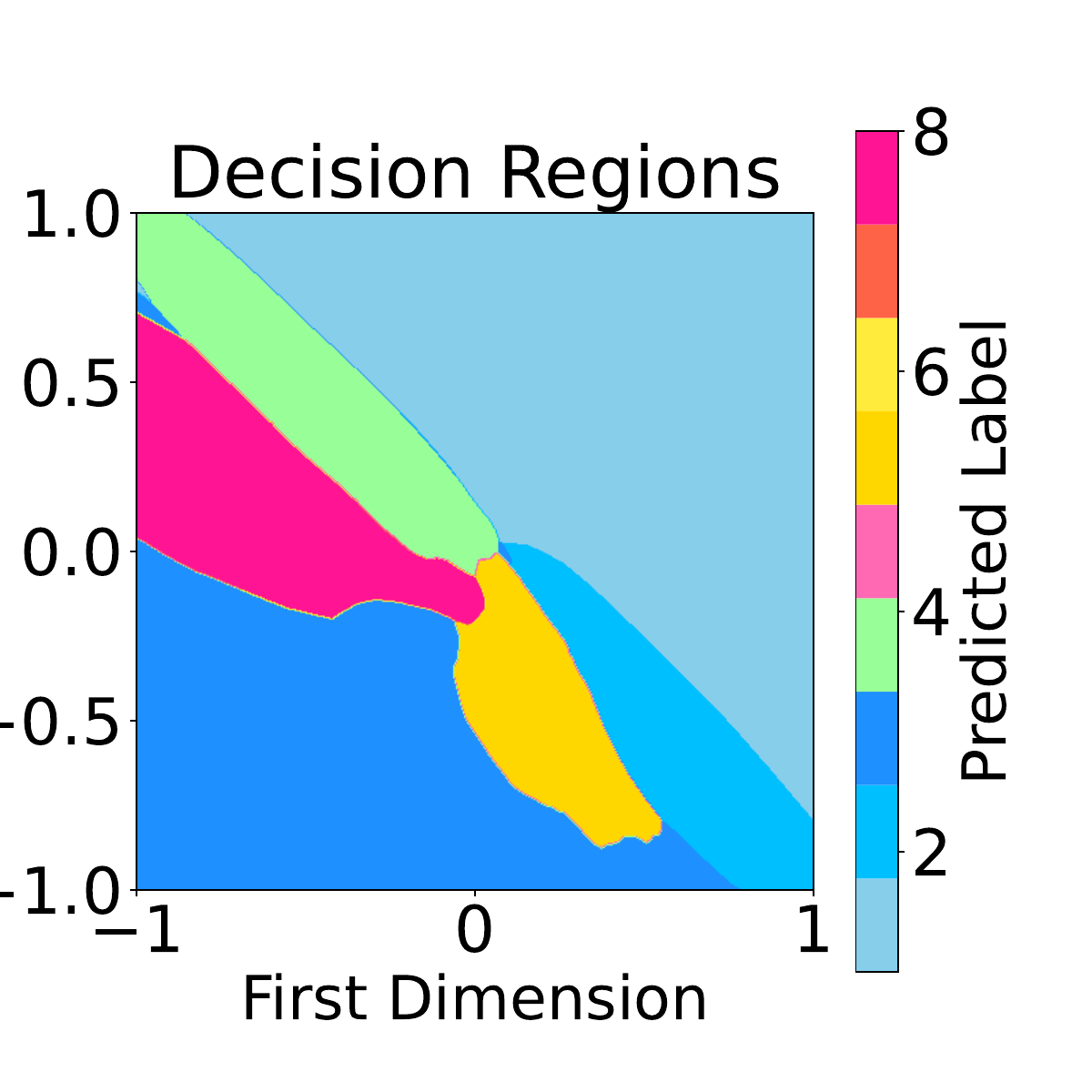}
    \caption*{(a) No regularization.}
  \end{subfigure} \hspace{5mm}
  \begin{subfigure}[t]{0.25\linewidth}
    \includegraphics[height=4cm, width=4cm]{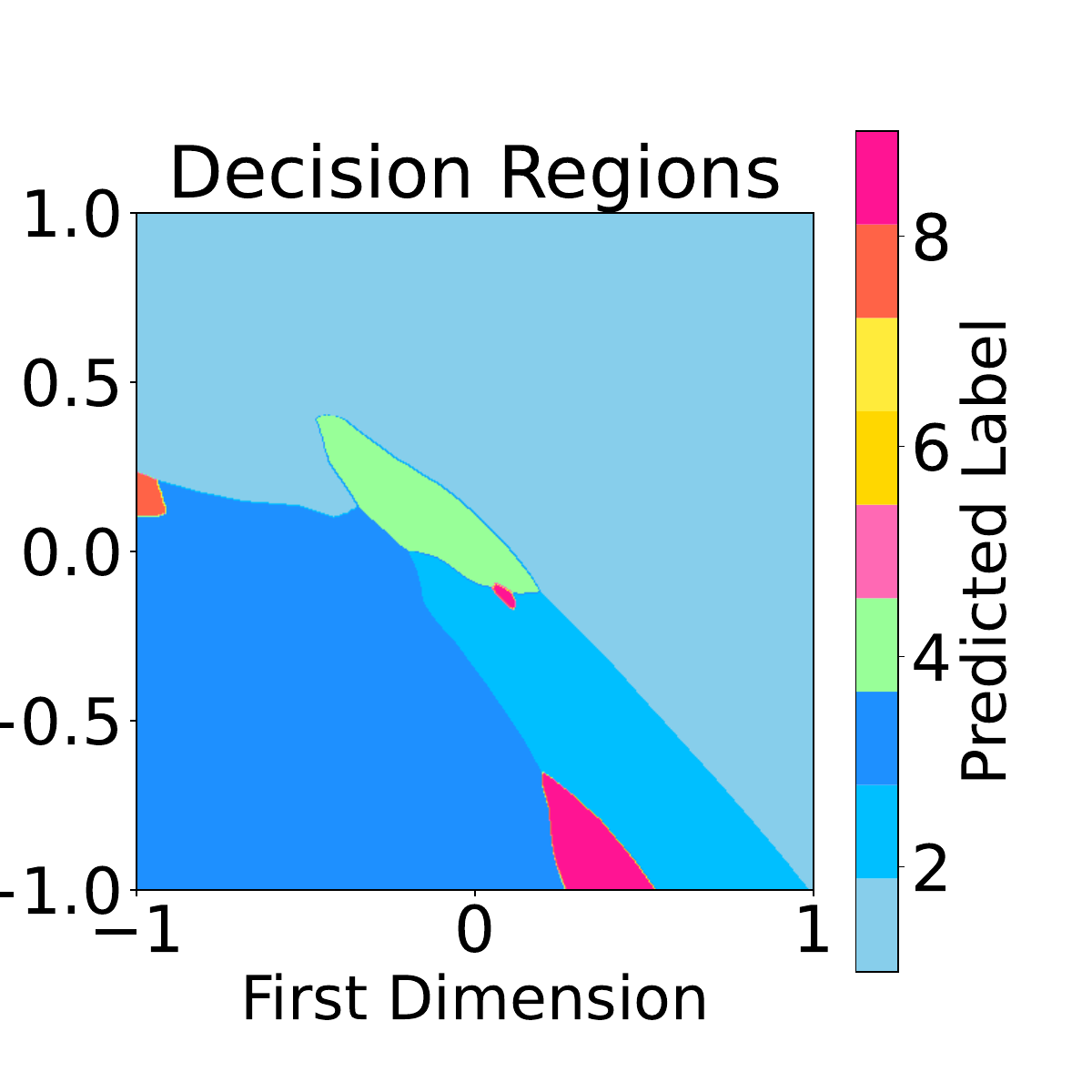}
    \caption*{(b) With $L^2$ regularization.}
  \end{subfigure} \hspace{5mm}
  \begin{subfigure}[t]{0.25\linewidth}
    \includegraphics[height=4cm, width=4cm]{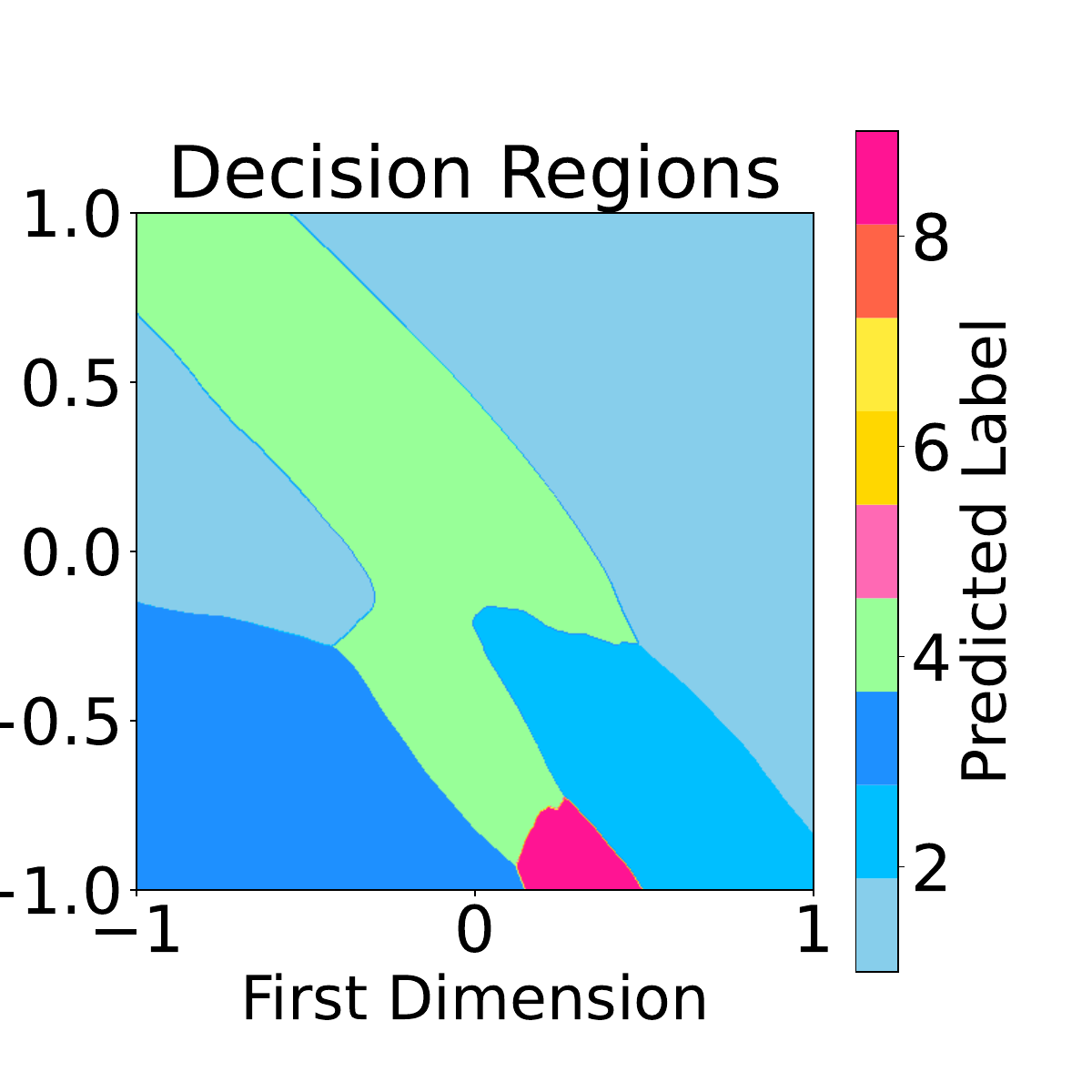}
    \caption*{(c) With $\mathcal{L}_{gm}$.}
  \end{subfigure} 
  \begin{subfigure}[t]{0.02\linewidth}
    \includegraphics[height=4cm, keepaspectratio]{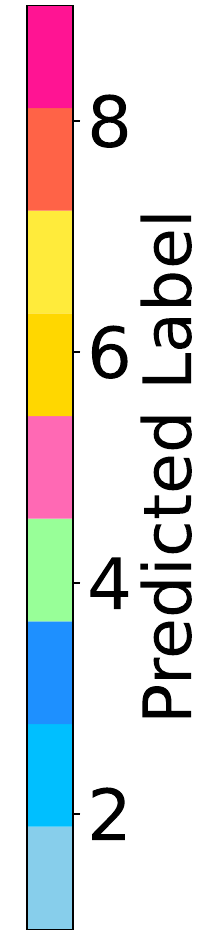} 
  \end{subfigure}

  \vspace*{-2.5mm}
  \caption{Visualization of input sensitivity for models trained with (a) no (b) $L^2$ (c) $\mathcal{L}_{gm}$ regularization. We randomly select a sample and introduce perturbations on a two-dimensional hyperplane, where different colors represent different labels, and green indicates the correct label.}
  \label{fig6}
  \vspace*{-4mm}
\end{figure*}

Results in Table~\ref{tab:2} indicate that our method provides a 3\% improvement across multiple NLU datasets with LLM backbone.
In addition, we select five commonly used backbone models in the visual domain and conduct experiments on fourteen datasets. The results are shown in Table~\ref{tab3}. It can be seen that our method is equally applicable to visual tasks, achieving a performance improvement of more than 1\% under the linear probe setting.

Furthermore, we visualize the sensitivity of different regularization terms to input perturbations, as illustrated in Figure~\ref{fig6}. 
Compared with other regularization methods, our loss function can increase the size of the region for correct decisions, thereby enhancing the model's robustness to input perturbations. 
We also carry out ablation studies and parameter sensitivity analyses, with results presented in Table~\ref{tab5} and Table~\ref{tab6}, which all demonstrate the effectiveness and robustness of LGM.

\section{Conclusion}
\label{sec:5}
In this paper, we investigate the random noise present in language model pre-training datasets, which is inevitable in real-world scenarios but receives little attention. We pre-train multiple GPT-2 models under varying noise levels and find that random noise has a minor impact on the pre-training loss. We then provide a theoretical explanation for this phenomenon and discover that our theory can elucidate the success of multilingual and multimodal models. Interestingly, we observe that slight noise can sometimes enhance a model's generalization ability. Then, building on the noisy model learning setup, we propose a novel local gradient matching loss. Extensive experiments across multiple datasets in both language and vision tasks, as well as with various backbone models, validate the effectiveness of our proposed method. We hope this work inspires more researchers to focus on data-centric AI.

\bibliographystyle{IEEEtran}
\bibliography{neurips_2025}

\newpage

\appendix

\section{Notations}
\label{sec:appa}

The commonly used notations and their descriptions are as follows.
\begin{table*}[h]
  \centering
  \begin{tabular}{|c|c|}
    \hline
    Notation                                                    & Description                                                                   \\ \hline
    $L$                                                         & context length                                                                \\ \hline
    $d$                                                         & embedding dimension                                                           \\ \hline
    $\mathcal{W}$                                               & vocabulary of words                                                           \\ \hline
    $V=|\mathcal{W}|$                                           & vocabulary size                                                               \\ \hline
    $\mathcal{X} = \cup_{i=1}^{L}\mathcal{W}^{i}$ & model input space                                                                     \\ \hline
    $\mathcal{H}$                                               & model space                                                                   \\ \hline
    $h:\mathcal{X} \rightarrow \mathbb{R}^V \in \mathcal{H}$    & language model                                                                \\ \hline
    $\Delta_A$                                                  & distribution defined on a discrete set $A$                                    \\ \hline
    $P^c \in \Delta_{\mathcal{X} \times \mathcal{W}}$           & distribution of clean data                                                    \\ \hline
    $P^n \in \Delta_{\mathcal{X} \times \mathcal{W}}$           & distribution of pure noise data                                               \\ \hline
    $P^m \in \Delta_{\mathcal{X} \times \mathcal{W}}$           & distribution of mixed noisy data                                              \\ \hline
    $\alpha$                                                    & proportion of noise in training data                                          \\ \hline
    $P_X$                                                       & marginal distribution of the joint distribution $P$                           \\ \hline
    $P_{\cdot \vert X}$                                         & conditional distribution of the joint distribution $P$                        \\ \hline
    ${\boldsymbol{p}}^{h}_{\cdot \vert x}(w)$                   & the $w$-th dimension of the probability distribution corresponding to $h(x)$  \\ \hline
    ${\text{supp}}(P^c)$                                        & support of distribution $P^c$                                                 \\ \hline
    $\mathcal{L}_{ntp}(P, h)$                                   & next-token prediction loss of model $h$ on the distribution $P$               \\ \hline
    $g_{\theta}: \mathbb{R}^d \rightarrow \mathbb{R}^C$         & downstream classification head              \\ \hline
    $\theta \in \Theta$                                         & parameters of $g$               \\ \hline
    $t \in \mathcal{T}$                                         & feature of downstream task data extracted by backbone model               \\ \hline
    $y \in \mathcal{Y}$                                         & label of downstream task data               \\ \hline
    $C=|\mathcal{Y}|$                                           & number of classes of the downstream task               \\ \hline
    $\ell(\hat{y}, y)$                                             & downstream task loss function, typically cross-entropy             \\ \hline
    $\mathcal{D}$                                               & joint distribution of downstream feature and label             \\ \hline
    $\mathcal{L}_{ce}(\mathcal{D}, g_{\theta})$                 & population-level loss with downstream data distribution $\mathcal{D}$ and head $g_{\theta}$               \\ \hline
    
  \end{tabular}
  \caption{\label{app-notation}
  Nomenclature.
  }
\end{table*}

\section{Proofs}
\label{app:proof}
\subsection{Explanation of Equation~(\ref{eq:huber})}
\label{app:huber}
Let $\mathcal{M}$ be a measurable space, and let $P_1$ and $P_2$ be probability measures defined on this space. We assume that $N_1$ samples are drawn from $P_1$ and $N_2$ samples from $P_2$. Define $\mu = \frac{N_1}{N_1 + N_2}$, so that $1 - \mu = \frac{N_2}{N_1 + N_2}$.

We aim to show that this collection of $ N_1 + N_2 $ samples can be regarded as drawn from a mixed distribution $$ P_3 = \mu P_1 + (1 - \mu) P_2 $$

First, define a new probability measure $ P_3 $ as $P_3(A) = \alpha P_1(A) + (1 - \alpha) P_2(A)$ for any measurable set $A \subseteq \mathcal{M}$. Here, $P_3$ is a convex combination of $P_1$ and $P_2$, and thus $P_3$ is also a valid probability measure \citep{gtm}.

For any measurable set $ A \subseteq \mathcal{M}$, we examine the probability that a single sample point falls in $A$ by law of total probability:
\begin{enumerate}[label=\textbullet]
  \item A sample from $P_1$ is selected with probability $\mu$, and within this case, the probability of landing in $ A $ is $ P_1(A) $.
  \item A sample from $P_2$ is selected with probability $1 - \mu$, and the probability of it falling in $A$ is $ P_2(A) $.
\end{enumerate}

Thus, the probability of any given sample point falling in $ A $ is
$$
\mu P_1(A) + (1 - \mu) P_2(A) = P_3(A)
$$


Since $N_1$ samples are drawn from $P_1$ and $N_2$ samples from $P_2$, these samples collectively follow the distribution $ P_3 $ as each individual sample's probability of being in any measurable set $ A $ is consistent with $ P_3(A) $. Therefore, drawing $N_1 + N_2$ samples in this manner is equivalent to drawing $N_1 + N_2$ samples from $ P_3 $.

\subsection{Proof of Proposition~\ref{pro1}}
\label{proof_pro1}

Before procedding to the proof, we first establish a useful lemma.
\begin{lemma}
\label{lemma1}
If Assumption \ref{assumption1} holds, then for any $h \in \mathcal{H}$, we have 
$$ \mathcal{L}_{ntp}(P^m, h) = \alpha \mathcal{L}_{ntp}(P^n, h)+(1-\alpha)\mathcal{L}_{ntp}(P^c, h) $$
\end{lemma}

\begin{proof}
  Let $x_i, \, i=1, 2, \ldots, |\mathcal{X}|$ denote all prefixes, and $w_j, \, j=1, 2, \ldots, V$ denote all tokens. For all $x \in \mathcal{X}$, by Equation~(\ref{eq:huber}), we have: 
  
  \begin{align}
  P^m_X(x) &= \sum_{j=1}^{V} P^m(x, w_j) = \sum_{j=1}^{V} \alpha P^n(x, w_j) + (1-\alpha) P^c(x, w_j) \notag \\
  &= \alpha \sum_{j=1}^{V} P^n(x, w_j) + (1-\alpha) \sum_{j=1}^{V} P^c(x, w_j) = \alpha P^n_X(x) + (1-\alpha) P^c_X(x) \label{eq:marginal}
  \end{align}
  This indicates that the marginal distribution possesses additivity. Consequently,
  \begin{align}
  \mathcal{L}_{ntp}(P^m, h) &= \mathbb{E}_{x \sim P^m_X} \mathbb{E}_{w \sim P^m_{\cdot | x}}-\log(\boldsymbol{p}^h_{\cdot | x}(w)) 
   = \sum_{i=1}^{|\mathcal{X}|} P^m_X(x_i) \cdot \mathbb{E}_{w \sim P^m_{\cdot | x_i}} -\log(\boldsymbol{p}^h_{\cdot | x_i}(w)) \notag \\
   &= \sum_{i=1}^{|\mathcal{X}|} [(1-\alpha) P^c_X(x_i) + \alpha P^n_X(x_i)]\cdot \mathbb{E}_{w \sim P^m_{\cdot | x_i}} -\log(\boldsymbol{p}^h_{\cdot | x_i}(w)) \tag*{(Equation~(\ref{eq:marginal}))} \\
   &= (1-\alpha)  \sum_{i=1}^{|\mathcal{X}|} P^c_X(x_i)\mathbb{E}_{w \sim P^m_{\cdot | x_i}} -\log(\boldsymbol{p}^h_{\cdot | x_i}(w)) + \alpha  \sum_{i=1}^{|\mathcal{X}|} P^n_X(x_i)\mathbb{E}_{w \sim P^m_{\cdot | x_i}} -\log(\boldsymbol{p}^h_{\cdot | x_i}(w)) \label{eq1}
   \end{align}
   The conditional distributions do not generally exhibit a linear relationship:
  $$P^m_{\cdot | x}(w|x) = \frac{P^m(x,w)}{P^m_X(x)}=\frac{(1-\alpha)P^c(x,w)+\alpha P^n(x,w)}{(1-\alpha)P^c_X(x)+\alpha P^n_X(x)} \ne  P^c_{\cdot | x}(w|x) \ne P^n_{\cdot | x}(w|x)$$
  However, if ${\text{supp}}(P^c) \cap {\text{supp}}(P^n)=\emptyset$, it immediately follows that:
  $$
    P^m_{\cdot | x}(w|x)=\frac{(1-\alpha)P^c(x,w)+\alpha P^n(x,w)}{(1-\alpha)P^c_X(x)+\alpha P^n_X(x)} = \begin{cases} 
      P^c_{\cdot | x}(w|x) & \text{if } (x,w) \in {\text{supp}}(P^c), \\
      P^n_{\cdot | x}(w|x) & \text{if } (x,w) \in {\text{supp}}(P^n).
    \end{cases} 
  $$
  Consequently,
  \begin{gather}
    \sum_{i=1}^{|\mathcal{X}|} P^c_X(x_i)\mathbb{E}_{w \sim P^m_{\cdot | x_i}} -\log(\boldsymbol{p}^h_{\cdot | x_i}(w)) 
    = \sum_{i=1}^{|\mathcal{X}|} P^c_X(x_i)\mathbb{E}_{w \sim P^c_{\cdot | {x_i}}} -\log(\boldsymbol{p}^h_{\cdot | x_i}(w)) 
    = \mathcal{L}_{ntp}(P^c, h) \label{eq2}
  \end{gather}
  Similarly,
  \begin{align}
  \sum_{i=1}^{|\mathcal{X}|} P^n_X(x_i)\mathbb{E}_{w \sim P^m_{\cdot | x_i}} -\log(\boldsymbol{p}^h_{\cdot | x_i}(w)) = \mathcal{L}_{ntp}(P^n, h) \label{eq3}
  \end{align}
  By substituting Equation~(\ref{eq2}) and Equation~(\ref{eq3}) into Equation~(\ref{eq1}), the proof is completed.
\end{proof}

Now we can prove Proposition~\ref{pro1}.
\begin{restatedproposition}{\ref{pro1}}
  Under Assumption~\ref{assumption1}, let $h^*$ be a model trained on $P^c$, with $\mathcal{L}_{ntp}(P^c,h^*)=-\log p_c$ and $\mathcal{L}_{ntp}(P^n,h^*)=-\log p_n$. When the model $h$ is trained on a mixed distribution $P^m$ which includes noise, it attempts to fit $P^n$, leading to an increase in the loss on the clean distribution $P^c$, such that $\mathcal{L}_{ntp}(P^c,h)=-\log(p_c-\epsilon)$ and $\mathcal{L}_{ntp}(P^n,h)=-\log(p_n+\epsilon/k)$ for some $\epsilon > 0$ ($k$ can be shown to be $\Omega(e^{\mathcal{L}_{ntp}(P^n,h)})$). Let $\eta = \alpha p_c - (1-\alpha)kp_n$.
We arrive at the following conclusions:

(1) If $\alpha \leq \frac{kp_n}{p_c + kp_n}$, then for any $0 < \epsilon <p_c$, we have $\mathcal{L}_{ntp}(P^m, h) \ge \mathcal{L}_{ntp}(P^m, h^*)$. This means that when $\alpha$ is sufficiently small, the global minimum on $P^m$ will not be affected by noise.

(2) If $\alpha > \frac{kp_n}{p_c + kp_n}$, then for $\epsilon < \eta$, it holds that $\mathcal{L}_{ntp}(P^m,h) < \mathcal{L}_{ntp}(P^m,h^*)$. This suggests that if $\alpha$ is large enough, the impact on the optimal hypothesis is at least as much as $\alpha p_c - (1-\alpha)kp_n$.

(3) When $\alpha < \frac{1}{3}$ and $k>\frac{\alpha(1-3\alpha)p_c}{(1-\alpha)(2-3\alpha)p_n}$, for $\epsilon \ge 3\eta$ we get $\mathcal{L}_{ntp}(P^m,h^*) < \mathcal{L}_{ntp}(P^m,h)$. Similarly, it can be shown that $\epsilon$ does not exceed $2\eta$ when $\alpha > \max(\frac{kp_n}{p_c + kp_n}, \frac{1}{2})$ and $k > \frac{(2\alpha-1)p_c}{2(1-\alpha)p_n}$. This indicates that when $k$ is sufficiently large, the effect of noise is at most $\mathcal{O}(\alpha p_c - (1-\alpha)kp_n)$.
\end{restatedproposition}

\begin{proof}
  We first establish that $k$ is $\Omega(e^{\mathcal{L}_{ntp}(P^n,h)})$, thereby ensuring that $\eta \ll \alpha p_c$. Note that
  \begin{align}
    \epsilon &= \frac{1}{e^{\mathcal{L}_{ntp}(P^c, h)}} - \frac{1}{e^{\mathcal{L}_{ntp}(P^c, h)}} = \frac{e^{\mathcal{L}_{ntp}(P^c, h)-\mathcal{L}_{ntp}(P^c, h)}-1}{e^{\mathcal{L}_{ntp}(P^c, h)}} \\
    \frac{\epsilon}{k} &= \frac{1}{e^{\mathcal{L}_{ntp}(P^n, h)}} - \frac{1}{e^{\mathcal{L}_{ntp}(P^n, h)}} = \frac{e^{\mathcal{L}_{ntp}(P^n, h)-\mathcal{L}_{ntp}(P^n, h)}-1}{e^{\mathcal{L}_{ntp}(P^n, h)}}
  \end{align}
  Therefore
  \begin{align}
   k = \frac{\epsilon}{\frac{\epsilon}{k}} &= e^{\mathcal{L}_{ntp}(P^n, h)-\mathcal{L}_{ntp}(P^c, h)} \cdot \frac{e^{\mathcal{L}_{ntp}(P^c, h)-\mathcal{L}_{ntp}(P^c, h)}-1}{e^{\mathcal{L}_{ntp}(P^n, h)-\mathcal{L}_{ntp}(P^n, h)}-1} \notag \\
   &> e^{\mathcal{L}_{ntp}(P^n, h)-\mathcal{L}_{ntp}(P^c, h)} \cdot \frac{\mathcal{L}_{ntp}(P^c, h)-\mathcal{L}_{ntp}(P^c, h)}{e^{\mathcal{L}_{ntp}(P^n, h)-\mathcal{L}_{ntp}(P^n, h)}} \notag \\
    &= 
    e^{\mathcal{L}_{ntp}(P^n, h)} \cdot \frac{\mathcal{L}_{ntp}(P^c, h)-\mathcal{L}_{ntp}(P^c, h)}{e^{\mathcal{L}_{ntp}(P^c, h)}} \label{eq4}
  \end{align}
  where $\frac{\mathcal{L}_{ntp}(P^c, h)-\mathcal{L}_{ntp}(P^c, h)}{e^{\mathcal{L}_{ntp}(P^c, h)}}$ only depends on $P^c$, $h$ and $h$. It is worth noting that when \( P^n \) is random noise, \( e^{\mathcal{L}_{ntp}(P^n, h)-\mathcal{L}_{ntp}(P^n, h)}-1 \) is close to $0$, which leads to \( k \) exceeding the lower bound established in Equation~(\ref{eq4}). Then:
  
  (1) If $\alpha \leq \frac{kp_n}{p_c + kp_n}$, we have:
  \begin{align}
    \mathcal{L}_{ntp}(P^m, h^*)-\mathcal{L}_{ntp}(P^m, h) &= (1-\alpha)(\mathcal{L}_{ntp}(P^c, h^*)-\mathcal{L}_{ntp}(P^c, h)) + \alpha (\mathcal{L}_{ntp}(P^n, h^*)-\mathcal{L}_{ntp}(P^n, h)) \notag \\
    &= (1-\alpha) \log \frac{p_c-\epsilon}{p_c} + \alpha \log \frac{p_n + \frac{\epsilon}{k}}{p_n} \label{eq:pro1_2} \\
    &\le (1-\alpha) \cdot \frac{-\epsilon}{p_c} + \alpha \cdot \frac{\frac{\epsilon}{k}}{p_n} \tag{$\log (1 + t) \le t$} \\
    &= \epsilon[\frac{(\alpha - 1)}{p_c} + \frac{\alpha}{kp_n}] = \epsilon \frac{\alpha p_c - (1-\alpha)kp_n}{k p_c p_n}
  \end{align}
  As $\alpha \leq \frac{kp_n}{p_c + kp_n} \iff \alpha p_c - (1-\alpha)kp_n \le 0$, for $\epsilon > 0$ we have $\mathcal{L}_{ntp}(P^m, h^*) \le \mathcal{L}_{ntp}(P^m, h)$.

  (2)when $\alpha > \frac{k p_n}{p_c + k p_n}$ and $\epsilon < \alpha p_c - (1 - \alpha)kp_n$, we have 
  \begin{align}
    \mathcal{L}_{ntp}(P^m, h^*)-\mathcal{L}_{ntp}(P^m, h) &= (1-\alpha) \log \frac{p_c-\epsilon}{p_c} + \alpha \log \frac{p_n + \frac{\epsilon}{k}}{p_n} \tag{Equation~(\ref{eq:pro1_2})}\\
    &\ge (1-\alpha) \frac{-\epsilon}{p_c-\epsilon} + \alpha \frac{\frac{\epsilon}{k}}{p_n+\frac{\epsilon}{k}} \tag{$\log t \ge 1 - \frac{1}{t}$} \\
    &= \epsilon (\frac{\alpha-1}{p_c-\epsilon} + \frac{\alpha}{kp_n+\epsilon}) \notag \\
    &= \frac{\epsilon}{(p_c-\epsilon)(kp_n+\epsilon)}[\alpha(p_c-\epsilon) - (1-\alpha)(kp_n+\epsilon)] \notag \\
    &= \frac{\epsilon}{(p_c-\epsilon)(kp_n+\epsilon)}[\alpha p_c - (1-\alpha)kp_n - \epsilon] \label{eq:lastpro1}
  \end{align}
  As $\epsilon < \alpha p_c - (1 - \alpha)kp_n < \alpha p_c < p_c$, by Equation~(\ref{eq:lastpro1}) we have $\mathcal{L}_{ntp}(P^m, h)-\mathcal{L}_{ntp}(P^m, h) > 0$. 

  (3) Let 
  \begin{equation}
  f(\epsilon) = (1-\alpha) \log \frac{p_c-\epsilon}{p_c} + \alpha \log \frac{p_n + \frac{\epsilon}{k}}{p_n} \overset{p_n'=kp_n}{=} (1-\alpha) \log \frac{p_c-\epsilon}{p_c} + \alpha \log \frac{p_n' + \epsilon}{p_n'}
  \end{equation}

  Take the derivative of $f(\epsilon)$:
  \begin{align}
    f'(\epsilon) &= (1-\alpha) \frac{-\frac{1}{p_c}}{1-\frac{\epsilon}{p_c}} + \alpha \frac{\frac{1}{p_n'}}{1+\frac{\epsilon}{p_n'}} = (1-\alpha) \frac{1}{\epsilon - p_c} + \alpha \frac{1}{p_n'+\epsilon} = \frac{[\alpha p_c-(1-\alpha)p_n']-\epsilon}{(p_c-\epsilon)(p_n'+\epsilon)}
  \end{align}

  Without loss of generality, assume $\eta > 0$, then $f(\epsilon)$ is monotonically increasing on $[0, \eta)$ and monotonically decreasing on  $(\eta, p_c)$. Therefore, to prove that $\mathcal{L}_{ntp}(P^m,h^*) < \mathcal{L}_{ntp}(P^m,h)$ for $\epsilon \geq 3\eta$, we only need to show $f(3\eta) < 0$ when $k>\frac{\alpha(1-3\alpha)p_c}{(1-\alpha)(2-3\alpha)p_n}$. Notice that 
  \begin{align}
    f(3\eta) &= (1-\alpha) \log(1-\frac{3\alpha p_c-3(1-\alpha)p_n'}{p_c}) + \alpha \log(1+\frac{3\alpha p_c-3(1-\alpha)p_n'}{p_n'}) \notag \\
    &= (1-\alpha) \log(1-3\alpha+\frac{3(1-\alpha)}{\frac{p_c}{p_n'}}) + \alpha \log(3\alpha-2+3\alpha \frac{p_c}{p_n'})
  \end{align}

  Let 
  \begin{equation}
    g_{3}(t)= (1-\alpha) \log(1-3\alpha+\frac{3(1-\alpha)}{t}) + \alpha \log(3\alpha-2+3\alpha t)
  \end{equation}
  Take the derivative:
  \begin{align}
    g'_3(t) &= (1-\alpha) \frac{1}{1-3\alpha+\frac{3(1-\alpha)}{t}} \frac{3(\alpha-1)}{t^2} + \alpha \frac{3\alpha}{3\alpha-2+3\alpha t} \\
    &= \frac{-3(1-\alpha)^2}{(1-3\alpha)t^2+(1-\alpha)t} + \frac{3\alpha^2}{3\alpha-2+3\alpha t} \\
    &= \frac{-3(1-\alpha)^2(3\alpha-2+3\alpha t) + 3\alpha^2 [(1-3\alpha)t^2+(1-\alpha)t]}{[(1-3\alpha)t^2+(1-\alpha)t](3\alpha-2+3\alpha t)} \\
    &= \frac{[\alpha t + (\alpha - 1)] [3 \alpha (1-3\alpha) t + 3(1-\alpha)(3\alpha - 2)]}{[(1-3\alpha)t^2+(1-\alpha)t](3\alpha-2+3\alpha t)}
  \end{align}
  First, consider the denominator. Since $\alpha < \frac{1}{3}$, it is clear that $(1-3\alpha)t^2 + (1-\alpha)t > 0$. Given that $t = \frac{p_c}{p_n'} > \frac{1-\alpha}{\alpha}$ (because $\eta > 0$), it follows that $3\alpha - 2 + 3\alpha t > 1 > 0$. Therefore, the denominator is always positive.
  Next, we consider the numerator. Since $\eta > 0$, it follows that $\alpha t + (\alpha - 1) > 0$. Therefore, when $t = \frac{p_c}{p_n'}=\frac{p_c}{k p_n} < \frac{(1-\alpha)(2-3\alpha)}{\alpha (1-3\alpha)}$, we have $g'_3(t) < 0$. This means that $g_3(t)$ is monotonically decreasing on $(\frac{1-\alpha}{\alpha}, \frac{(1-\alpha)(2-3\alpha)}{\alpha (1-3\alpha)}]$. Consequently, $f(3\eta) = g_3(t) \le g_3\left(\frac{1-\alpha}{\alpha}\right) = 0$.

  Following the same line of reasoning, when $\alpha > \frac{1}{2}$, we have
  \begin{align}
    f(2\eta) &= (1-\alpha) \log(1-\frac{2\alpha p_c-2(1-\alpha)p_n'}{p_c}) + \alpha \log(1+\frac{2\alpha p_c-2(1-\alpha)p_n'}{p_n'}) \notag \\
    &= (1-\alpha) \log(1-2\alpha+\frac{2(1-\alpha)}{\frac{p_c}{p_n'}}) + \alpha \log(2\alpha-1+2\alpha \frac{p_c}{p_n'})
  \end{align}

  Let
  \begin{equation}
    g_{2}(t)= (1-\alpha) \log(1-2\alpha+\frac{2(1-\alpha)}{t}) + \alpha \log(2\alpha-1+2\alpha t)
  \end{equation}
  Take the derivative:
  \begin{align}
    g'_2(t) &= (1-\alpha) \frac{1}{1-2\alpha+\frac{2(1-\alpha)}{t}} \frac{2(\alpha-1)}{t^2} + \alpha \frac{2\alpha}{2\alpha-1+2\alpha t} \\
    &= \frac{-2(1-\alpha)^2}{(1-2\alpha)t^2+2(1-\alpha)t} + \frac{2\alpha^2}{2\alpha-1+2\alpha t} \\
    &= \frac{2(1-2\alpha)(\alpha t + 1-\alpha)^2}{[(1-2\alpha)t^2+2(1-\alpha)t](2\alpha-1+2\alpha t)}
  \end{align}
  Therefore, when $\frac{1-\alpha}{\alpha} < t < \frac{2(1-\alpha)}{2\alpha-1}$, we have $g'_2(t) < 0$, which implies that $f(2\eta) < 0$.

\end{proof}




\subsection{Justification of Proposition~\ref{pro1}}
\label{app:explain}

\begin{figure*}
  \centering
  \begin{subfigure}{0.37\linewidth}
    \includegraphics[height=4.5cm, width=\linewidth]{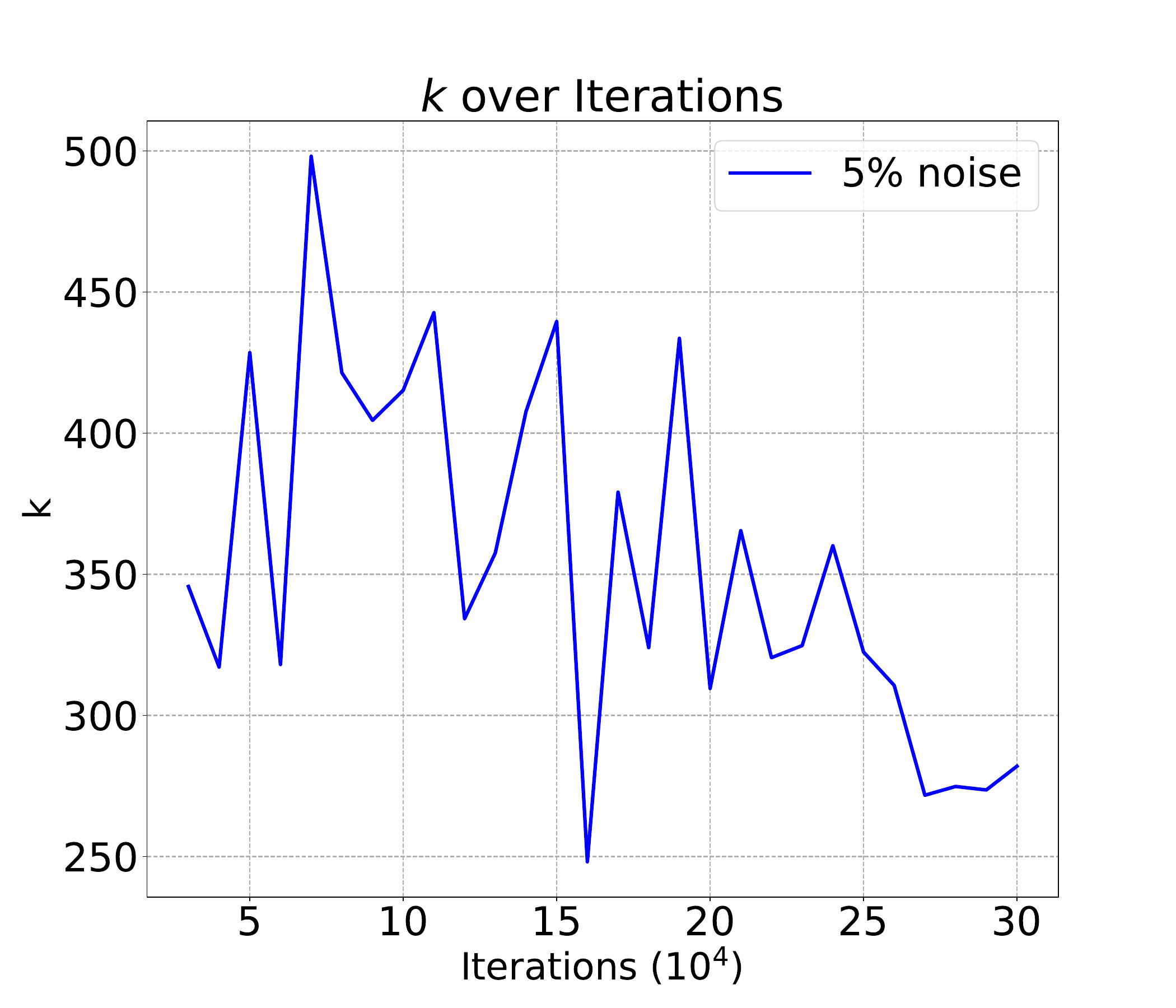}
    \caption*{(a)}
  \end{subfigure}
  \begin{subfigure}{0.55\linewidth}
    \includegraphics[height=4.5cm, width=\linewidth]{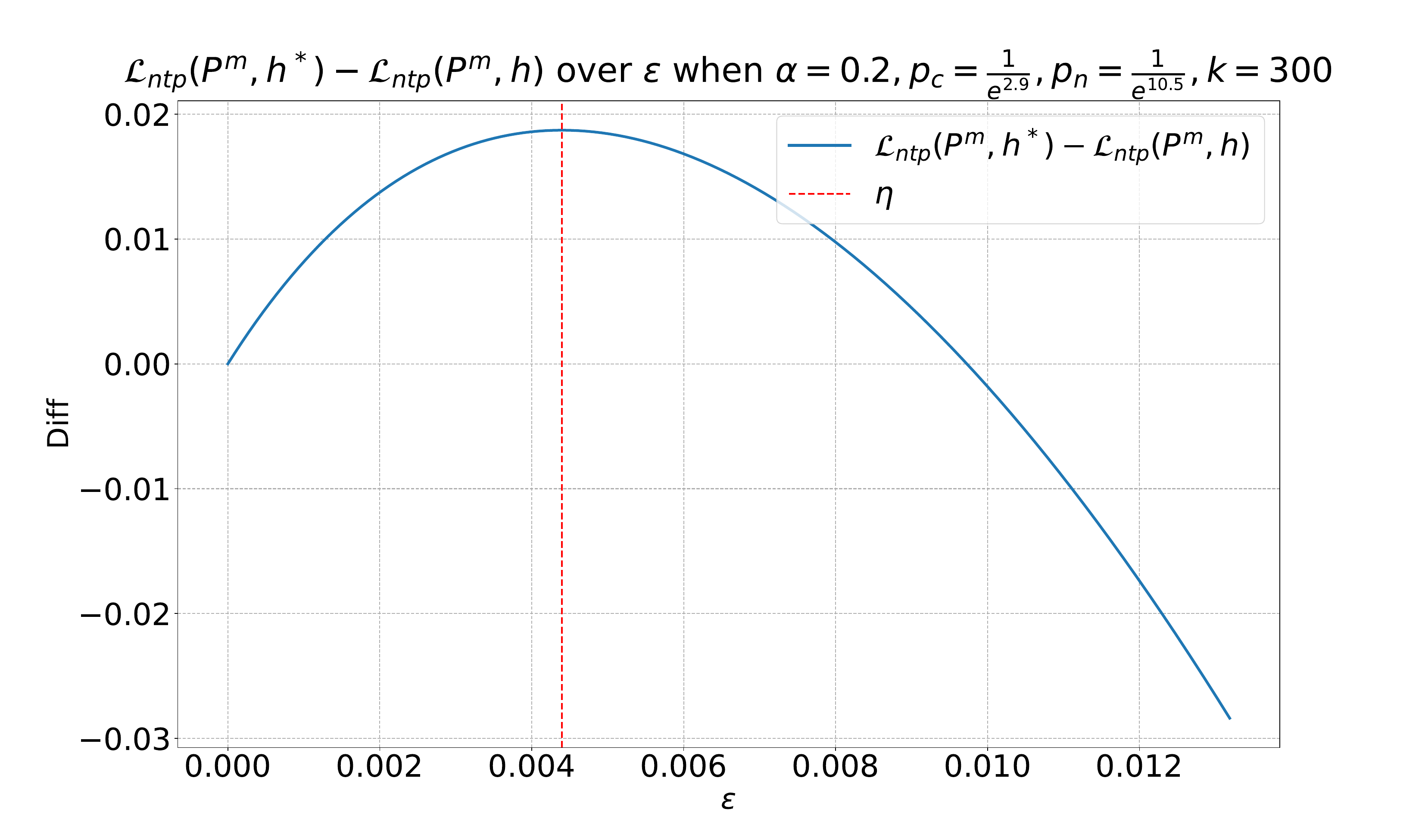}
    \caption*{(b)}
  \end{subfigure}

  \caption{Visualization of $k$ and $\mathcal{L}_{ntp}(P^m, h^*) - \mathcal{L}_{ntp}(P^m, h)$. (a) The trend of $k$ as it changes with training, plotted using the model trained on $P^c$ as $h^*$. (b) Visualization of $\mathcal{L}_{ntp}(P^m, h)$ when the parameter settings are consistent with the experiment.}
  \label{fig:app:k}

\end{figure*}

We plot the trend of $k$ in Figure~\ref{fig:app:k}(a). We compare checkpoints trained for the same iterations on both $P^c$ and $P^m$, where $p_c$ is calculated based on the loss of the model trained on $P^c$, and $p_n$ is determined by the loss of a model trained for 10,000 iterations on $P^m$ when evaluated on $P^n$. It can be observed that the value of $k$ corresponding to random noise is significantly greater than one, which supports the rationality of the assumption made in Proposition~\ref{pro1}.

On the other hand, to extend the proposed theory beyond uniformly distributed random noise (for instance, in multilingual models or Gaussian noise), it is necessary to ensure that $k$ does not become too small in these scenarios. This means that $\mathcal{L}_{ntp}(P^n,h^*)=-\log p_n$ should not be close to $\log V$. One trivial way to increase $p_n$ is to decrease V, the size of vocabulary. Apart from this, we provide two lines of reasoning to justify why $p_n$ can be made large:

(1) Numerous studies on compressing large language models, such as pruning \citep{llmjz1, llmjz2, llmjz3}, quantization \citep{llmlh1, llmlh2, llmlh3}, and distillation \citep{llmzl1, zl2}, have demonstrated that there exists a significant amount of redundancy within the parameters of large models. Therefore, we could first train a model on $P^c$ and then compress it, fine-tuning the surplus parameters on $P^n$. This approach would allow us to improve $p_n$ without altering $p_c$.

(2) A small proportion of data corresponding to $P^n$ can be introduced into $P^c$, making sure that $\alpha$ is extremely small. According to domain adaptation theory \citep{ben-david}, this would only slightly increase $\mathcal{L}_{ntp}(P^c)$. However, existing results \citep{mlbt1, mlbert1, mlbert2} indicate that pre-trained models like BERT or GPT on English text can exhibit strong multilingual capabilities with just a very limited amount of data. Consequently, compared to a model trained solely on $P^c$, the resulting model has a minor difference in $p_c$ but a relatively higher $p_n$.

Both thought experiments above demonstrate that there exist a lot of models within the parameter space $\mathcal{H}$ can perform well on $P^c$ while yielding non-trivial outcomes on $P^n$. Thus, we can ensure that models trained on mixed data distributions will have a sufficiently large $k$.

Additionally, in Figure~\ref{fig:app:k}(b), we illustrate how $\mathcal{L}_{ntp}(P^m)$ varies with changes in $\epsilon$, under settings identical to those used during pre-training. The results depicted in the figure are consistent with our theoretical derivations.

\subsection{Omitted Details in Section~\ref{sec:4.2}}
\label{app:4.2}

\begin{definition}[$\beta$-smooth \citep{unlabeled}]
  \label{def:smooth}
  A loss function $\ell(g_{\theta}(t), y)$ is $\beta$-smooth, if for any $(t, y) \in \mathcal{T} \times \mathcal{Y}$ and any $\theta, \theta' \in \Theta$,
  \begin{equation}
    ||\nabla_{\theta} \ell(g_{\theta}(t), y ) - \nabla_{\theta'} \ell(g_{\theta'}(t), y)||_2 \le \beta ||\theta - \theta'||_2
  \end{equation}
\end{definition}

\begin{definition}[$\rho$-input flatness]
  \label{def:flat}
  The $\rho$-input flatness $R_{\rho}(\theta)$ of loss function $ \ell(g_{\theta}(t), y )$ is defined as:
  \begin{equation}
    R_{\rho}(\theta) = \mathbb{E}_{(t, y) \sim \mathcal{D}} \sup_{\delta' \in B(0, \rho)}  \ell(g_{\theta}(t + \delta'), y ) -  \ell(g_{\theta}(t), y )
  \end{equation}
  where $B(0, \rho) = \{\delta' : ||\delta'||_2 < \rho\}$ is an open ball.
\end{definition}

\begin{lemma}
  \label{lemma2}
  If the loss function $\ell(g_{\theta}(t), y)$ is $\beta$-smooth, then 
  \begin{equation}
    ||\nabla_{\theta} \ell(g_{\theta}(t), y)||^2_2 \le 4 \beta \ell(g_{\theta}(t), y)
  \end{equation}
\end{lemma}

\begin{proof}
  See Lemma 3.1 in \citep{smoothness}.
\end{proof}

\begin{restatedproposition}{\ref{pro2}}
  Suppose $\ell(g_{\theta}(t), y)$ is $\beta$-smooth with $\rho$-input flatness $R_{\rho}(\theta)$, for any $\theta \in \Theta$:
  \begin{equation}
  \mathcal{L}_{gm}(\theta) \le 2 \beta + 2 \mathcal{L}_{ce}(\mathcal{D}, g_{\theta}) + R_{\rho}(\theta)
  \end{equation}
\end{restatedproposition}

\begin{proof}
  \begin{align}
    \mathcal{L}_{gm}(\theta) &= || \mathbb{E}_{(t, y) \sim {\mathcal{D}}} \nabla_{\theta}  \ell(g_{\theta}(t), y ) - \mathbb{E}_{(\hat{t}, y) \sim \hat{\mathcal{D}}} \nabla_{\theta}  \ell(g_{\theta}(\hat{t}), y ) ||_2 \\
    &\le || \mathbb{E}_{(t, y) \sim {\mathcal{D}}} \nabla_{\theta}  \ell(g_{\theta}(t), y )||_2 + ||\mathbb{E}_{(\hat{t}, y) \sim \hat{\mathcal{D}}} \nabla_{\theta}  \ell(g_{\theta}(\hat{t}), y ) ||_2 \tag{Triangle Inequality} \\
    &\le \mathbb{E}_{(t, y) \sim {\mathcal{D}}} ||\nabla_{\theta}  \ell(g_{\theta}(t), y )||_2 + \mathbb{E}_{(\hat{t}, y) \sim \hat{\mathcal{D}}} ||\nabla_{\theta}  \ell(g_{\theta}(\hat{t}), y ) ||_2 \tag{Jensen's Inequality} \\
    &\le \mathbb{E}_{(t, y) \sim {\mathcal{D}}} 2 \sqrt{\beta  \ell(g_{\theta}(t), y )} + \mathbb{E}_{(\hat{t}, y) \sim \hat{\mathcal{D}}} 2 \sqrt{\beta  \ell(g_{\theta}(\hat{t}), y )} \tag{Lemma~\ref{lemma2}} \\
    &\le \mathbb{E}_{(t, y) \sim {\mathcal{D}}} (\beta+ \ell(g_{\theta}(t), y )) + \mathbb{E}_{(\hat{t}, y) \sim {\hat{\mathcal{D}}}} (\beta+ \ell(g_{\theta}(\hat{t}), y )) \tag{AM-GM Inequality} \\
    &= 2\beta + 2 \mathbb{E}_{(t, y) \sim {\mathcal{D}}}  \ell(g_{\theta}(t), y ) + (\mathbb{E}_{(\hat{t}, y) \sim {\hat{\mathcal{D}}}}  \ell(g_{\theta}(\hat{t}), y ) - \mathbb{E}_{(t, y) \sim {\mathcal{D}}}  \ell(g_{\theta}(t), y )) \\
    &\le 2\beta + 2 \mathcal{L}_{ce}(\mathcal{D}, g_{\theta}) + R_{\rho}(\theta)
  \end{align}
  where the last inequality holds because $\hat{t} - t \in B(0, \rho)$.
\end{proof}

\section{Detailed Related Works}
\label{app:relatedwork}
\textbf{Data selection for language model training.} LLMs \citep{ydc2, zxw1, zxw2, xyx1, xyx2, wzy1, cmr1} have fundamentally change the landsape of current AI research \citep{xyx3, cmr2, zxw3, zxw4, ydc3, ydc4}. Large text corpora form the backbone of language models, with data quality being fundamental to their success and safety \citep{yyg1, yyg2}. \citep{wimbd} conducted a systematic analysis of open-source text datasets such as The Pile \citep{pile} (used to train Pythia), C4 \citep{c4} (used to train T5) and RedPajama (used to train LLaMA), revealing that they contain a significant amount of duplicate, toxic, synthetic, and low-quality content. Therefore, it is of great importance to thoroughly understand the impact of low-quality data within these pre-training datasets on the model's performance, reliability, and safety. \citep{allenphysics, allen33} systematically investigated the effect of low-quality data and found that such data can significantly reduce the model's knowledge capacity, sometimes by up to 20 times. Another research direction primarily focuses on the synthetic data of large language models, specifically examining the impacts of using data generated by LLMs for recursive training. The study by \citep{naturemc} was the first to explore this issue and introduced the concept of "model collapse", indicating that recursive training can lead to the loss of information in tail tokens, ultimately resulting in the model producing nonsensical content. \citep{colmmc} mainly provided a theoretical explanation for why model collapse occurs, supporting their arguments with experimental evidence. Consequently, the importance of data selection cannot be overstated. Given that data selection is an NP-hard problem in terms of combinatorial optimization \citep{xmy}, numerous heuristic algorithms have been proposed to expedite the process. \citep{pretrainguide} provided a comprehensive study on pre-training data selection and optimal ratios, offering practical recommendations. \citep{xzk1} proposed dataset pruning, an approach that assesses the impact of omitting training samples on model generalization and creates a minimal training subset with a controlled generalization gap. \citep{gptinfluence} evaluated the impact of individual training samples on the dynamics of GPT model training. \citep{quality_naacl} introduced the Instruction-Following Difficulty metric to assess the quality of instruction-tuning data. \citep{xie2023data} employed importance resampling for data selection. \citep{doremi, beyond} advocated for optimizing data composition and diversity. Despite these notable studies on data selection, they generally acknowledge that dataset noise degenerates model performance but lack a detailed understanding of how and to what extent, particularly in the case of random noise which is inevitable in large-scale datasets. Although \citep{gibberish} investigated the influence of gibberish input, the random noise within the pre-training dataset is still underexplored. This paper aims to bridge the gap.

\textbf{Learning from Noisy Distributions.} The majority of machine learning algorithms assume that training and test samples are independently and identically distributed (i.i.d.), a condition that is often not met in real-world scenarios. For instance, LLMs are pre-trained on datasets with all kinds of noise while their performance is evaluated by the user whose distribution is usually clean and meets real-world scenarios, which violates the i.i.d. assumption. Domain adaptation \citep{yzq, mlc1, zcy1} addresses this issue when the distribution of the training data differs from that of the test data. Although domain adaptation methods attempt to reduce the statistical distribution discrepancy \citep{zkd1, kkw1} or employ adversarial training \citep{jjl1, iosda} to minimize the gap between source and target domains, they typically require access to unlabeled test data under a semi-supervised learning setup, which is impractical for LLM training. Another reason domain adaptation cannot be directly applied here is that domain adaptation theory \citep{ben-david} focuses on the performance of a model trained on one distribution when it is applied to another different but related distribution. This kind of bounds can be easily derived by Lemma~\ref{lemma1}. However, what we aim to investigate here is the extent of performance loss when comparing a model trained on one distribution (noisy dataset) to a model trained on another distribution (clean dataset).

Apart from domain adaptation, there has been extensive research directly investigating noisy training sets. Noisy label learning \citep{ls1, ls2} have explored the impact of incorrect labels on model performance. Regarding input feature noise, \citep{smoothgrad} added perturbations to individual image inputs to enhance model interpretability, and \citep{purenoise} added white noise image into the training dataset to tackle the class imbalance problem. However, most of these efforts have concentrated on image classification and do not consider the pre-training paradigm.

\textbf{Fine-tuning Pre-trained Models.} The approach of initially pre-training model weights on large-scale datasets and subsequently fine-tuning them with downstream data has become the de facto standard in the fields of computer vision \citep{yyq, vargpt, cmr3, zxw5}, natural language processing \citep{whyhelp, ptft3, zxw6, cmr4} and audio processing \citep{uniaudio, zxw7, zxw8}. For instance, \citep{noisebert1, noisebert2} proposed enhancing the performance of models by increasing their resistance to minor perturbations in intermediate layers. Meanwhile, \citep{smart} improved model robustness by adding regularization terms. Besides full-parameter fine-tuning, numerous parameter-efficient fine-tuning algorithms have been extensively studied. \citep{adapter} introduced adapters into the original model architecture, optimizing only these parameters during fine-tuning. \citep{coop, cocoop} efficiently fine-tuned CLIP models \citep{clip} using learnable soft prompts. \citep{lora} optimized models through learning low-rank residual weights. These methods achieved performance close to that of full-parameter fine-tuning while maintaining the generalization ability of the original models. However, they all require access to the model's weights and loading them into GPU memory, which can be challenging for today's large models, especially when state-of-the-art models' parameters are not publicly available. Therefore, in this paper, we follow the NML setup and explore efficient ways to fine-tune the downstream task head under a black-box scenario.

\textbf{Implicit Regularization and Sharpness-aware Minimization.} Achieving good generalization in neural networks optimized using gradient descent algorithms has long been a research focus in deep learning theory. \citep{gradnorm2} explored the properties of stochastic gradient descent (SGD), finding that SGD implicitly constrains the gradient norm. Based on this observation, Sharpness-aware minimization (SAM) \citep{sharpness, sharpness0, sharpness2, xzk2} improves generalization by incorporating the gradient norm as a regularization term. Our method can be seen as drawing inspiration from SAM but differs in that our optimization objective is the model's resilience to input noise rather than seeking flat minima in the parameter space.

\section{Experiments in Section~\ref{sec:3}}
\label{sec:mer}

\subsection{Pre-training Dataset}


\textbf{OpenWebText Dataset.} The OpenWebText dataset \citep{openwebtext} is a large-scale corpus of English text data, developed to serve as an open-access alternative to proprietary dataset WebText that is utilized by OpenAI for training their GPT-2 models. This dataset originates from the analysis of outbound links clicked on Reddit, undergoing multiple stages of filtering to exclude non-English content, duplicate entries, copyrighted materials, and texts lacking in quality. These links generally direct to web pages available to the public, often shared or debated on Reddit, thereby covering a broad spectrum of subjects that mirror online popular interests and discussions. The dataset includes roughly 18 million documents, amounting to about 20GB of compressed plain text data in uint16 format. Since measures have been implemented to ensure the dataset's reliability by filtering out unsuitable content, we consider it a clean and noise-free dataset.



\subsection{Training Details of GPT-2}
\label{app:hyper1}
Our work is based on the source code of nanoGPT\footnote{\url{https://github.com/karpathy/nanoGPT}}. Specifically, we utilized the GPT-2 tokenizer with vocabulary size $V=50256$ to process the OpenWebText dataset, and then appended randomly generated noise to the end of the training set before commencing training. The model's context length is set to \(L = 1024\), with an embedding dimension $d=768$. The GPT-2 model consists of 12 self-attention layers, totaling approximately 124 million parameters. For optimization, we employed AdamW \citep{adamw, adam1} with a learning rate of 6e-4, weight decay of 0.1, and \(\beta\) values of 0.9 and 0.95 for \(\beta_1\) and \(\beta_2\), respectively. A cosine annealing scheduler was used to gradually adjust the learning rate down to 6e-5. We configured the batch size to 16, with a gradient accumulation step of 40, allowing each iteration to process 655,360 tokens (16 * 40 * 1024). Training proceeded for a total of 300,000 iterations.

\subsection{Experiments on Other Text Corpus}
\label{arxiv}
\begin{figure*}[t]
  \centering
  \begin{subfigure}{0.23\linewidth}
    \includegraphics[height=2.5cm, width=\linewidth]{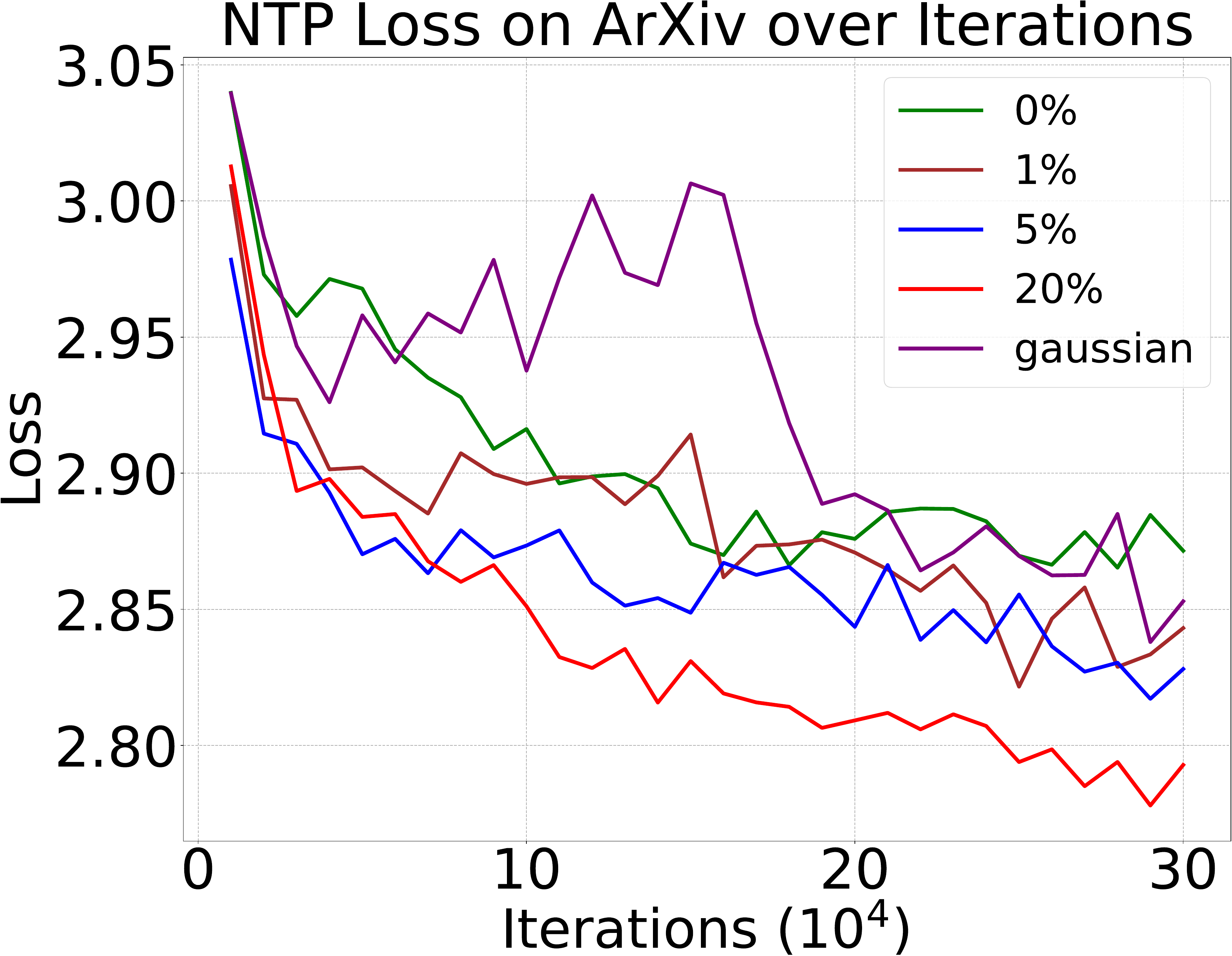}
    \caption*{(a)}
  \end{subfigure} \hfill
  \begin{subfigure}{0.24\linewidth}
    \includegraphics[height=2.5cm, width=\linewidth]{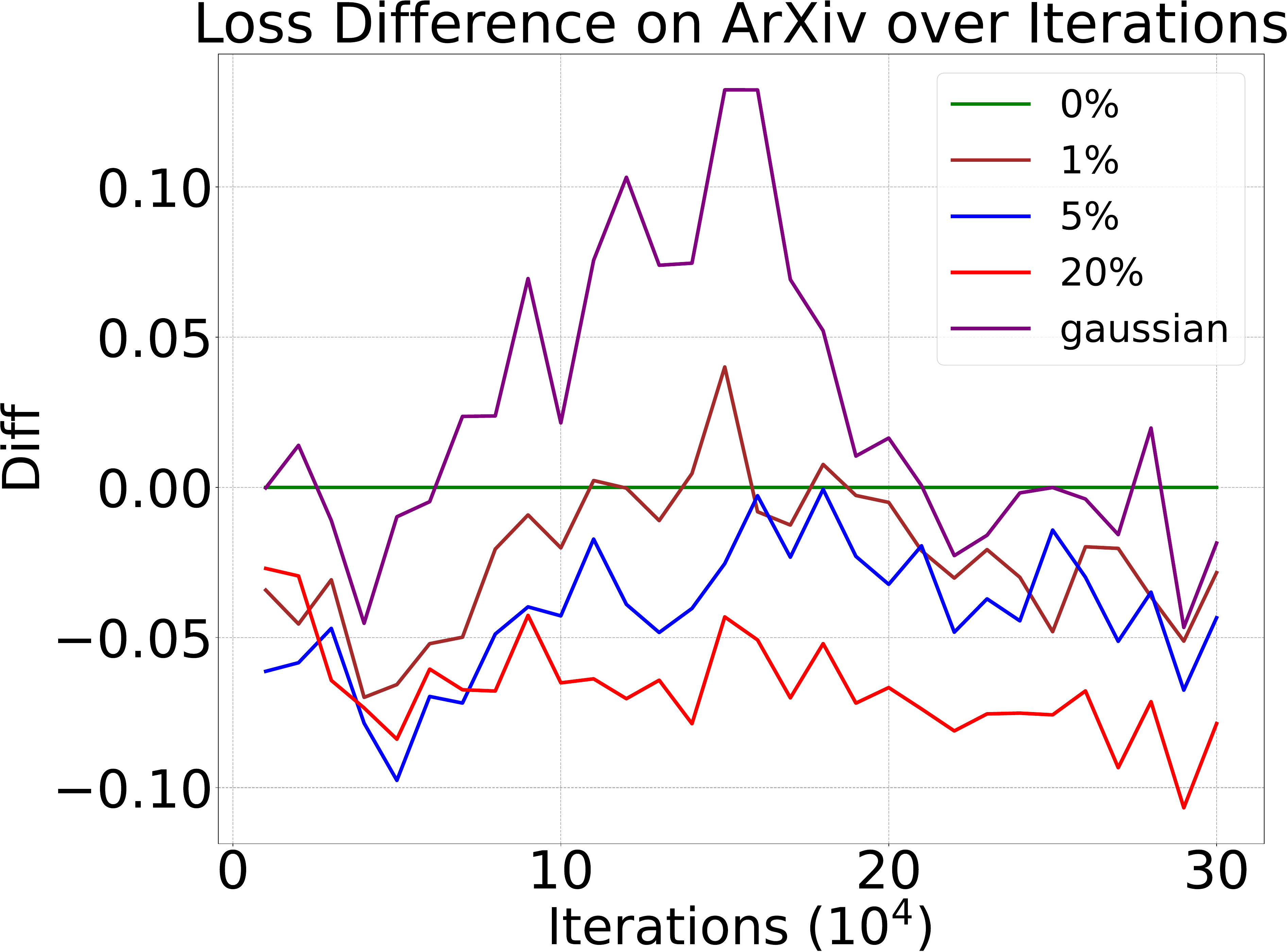}
    \caption*{(b)}
  \end{subfigure} \hfill
  \begin{subfigure}{0.23\linewidth}
    \includegraphics[height=2.5cm, width=\linewidth]{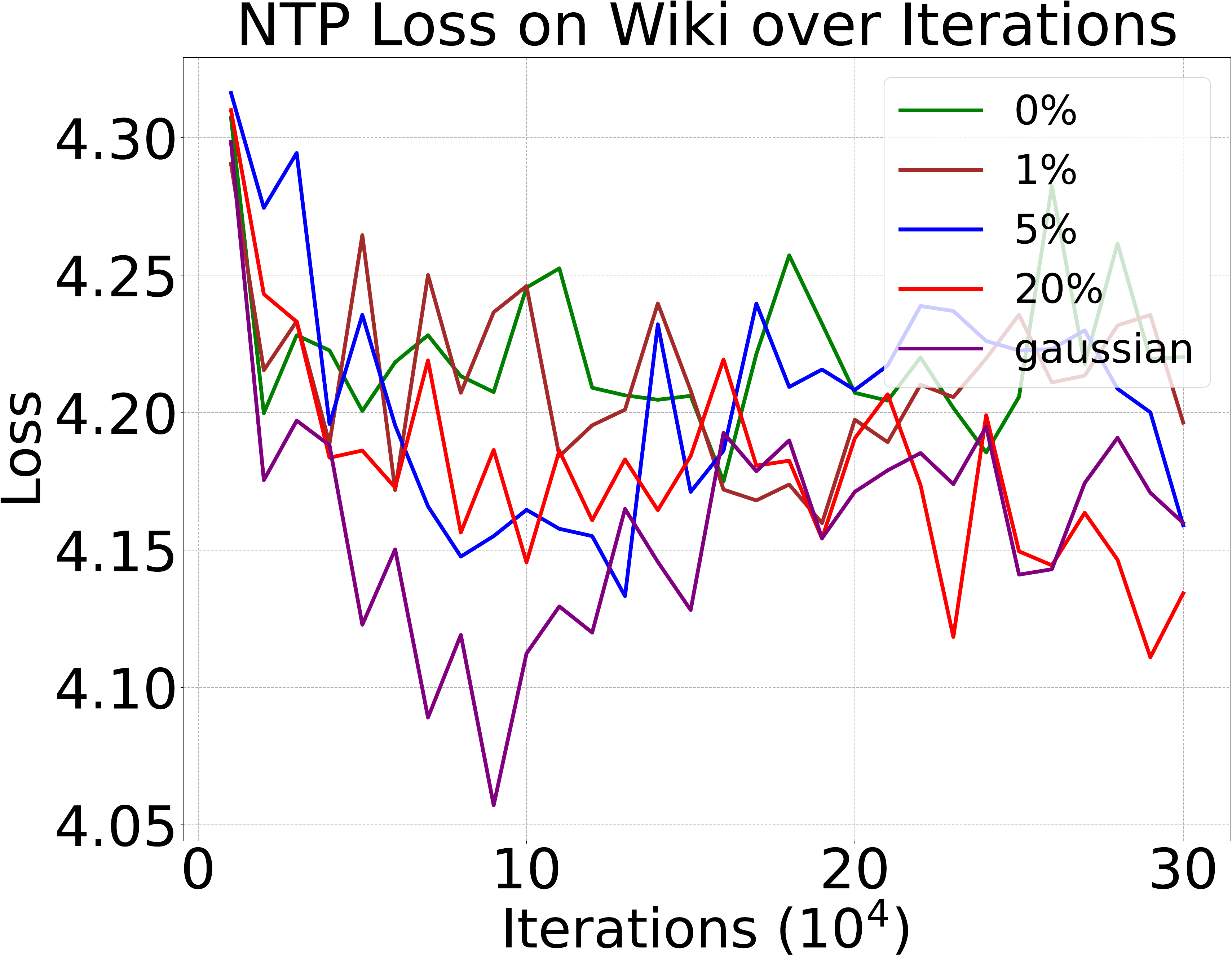}
    \caption*{(c)}
  \end{subfigure} \hfill
  \begin{subfigure}{0.24\linewidth}
    \includegraphics[height=2.5cm, width=\linewidth]{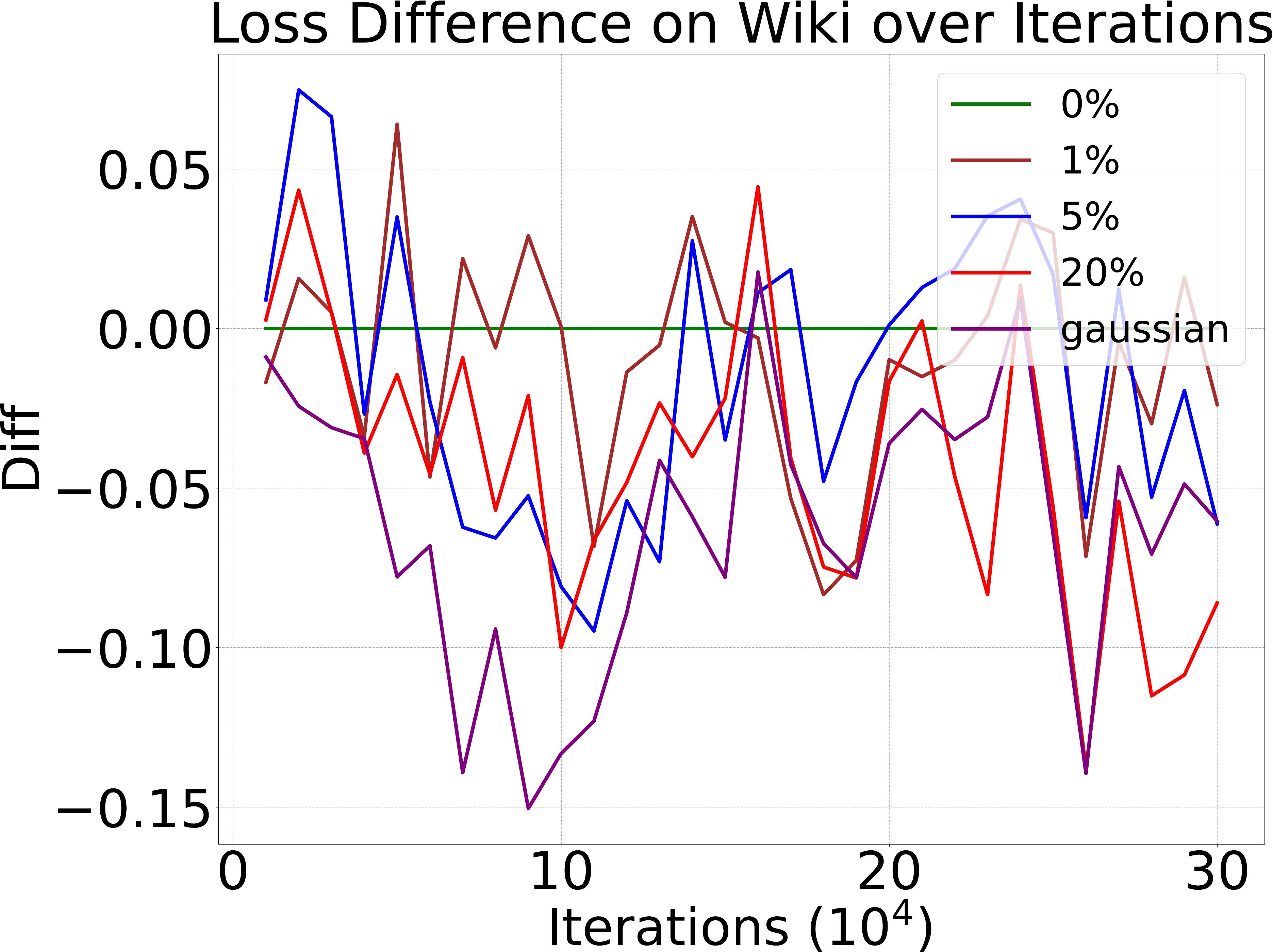}
    \caption*{(d)}
  \end{subfigure}
  \vspace*{-2.5mm}
  \caption{Loss and its difference across different types and levels of noise within the ArXiv and Wikipedia corpora.}
  \label{fig:4}
  \vspace*{-4mm}
\end{figure*}
To further investigate the impact of random noise on model generalization, we evaluate the next-token prediction loss of the trained models on data crawled from arXiv and Wikipedia. The results are illustrated in Figure~\ref{fig:4}. Surprisingly, models trained with added noise outperformed those trained on $P^c$. This counterintuitive finding aligns with previous work in visual domains \citep{purenoise}, suggesting that incorporating random noise into training sets might enhance model robustness. Additionally, we observe that the performance of models subjected to Gaussian noise varies across different datasets. These observations warrant further investigation.

\subsection{Synthetic Results about Multilingual Models}
\label{app:multilin}
\begin{figure*}
  \centering
  \begin{subfigure}{0.47\linewidth}
    \includegraphics[height=4cm, width=\linewidth]{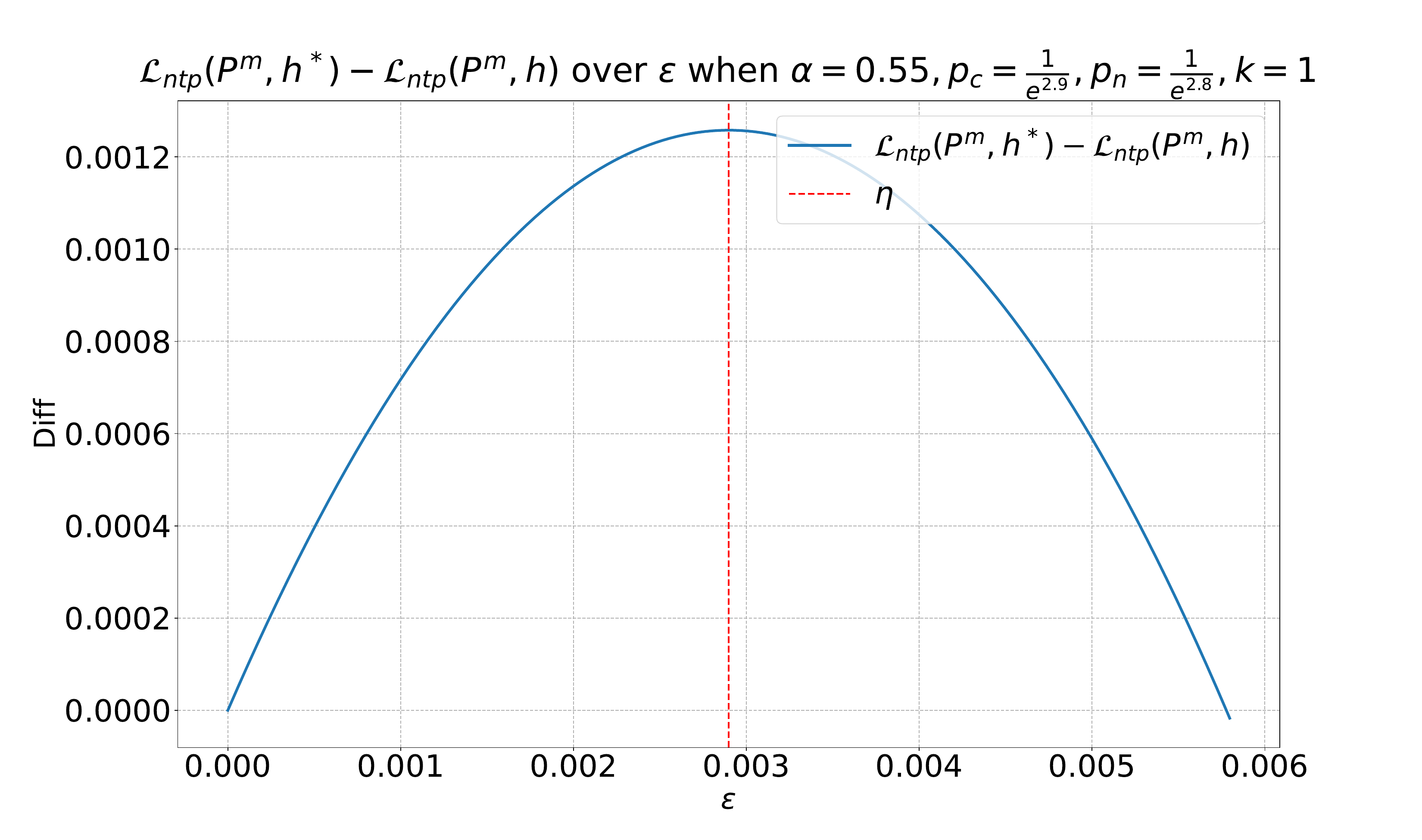}
    \caption*{(a)}
  \end{subfigure}
  \begin{subfigure}{0.47\linewidth}
    \includegraphics[height=4cm, width=\linewidth]{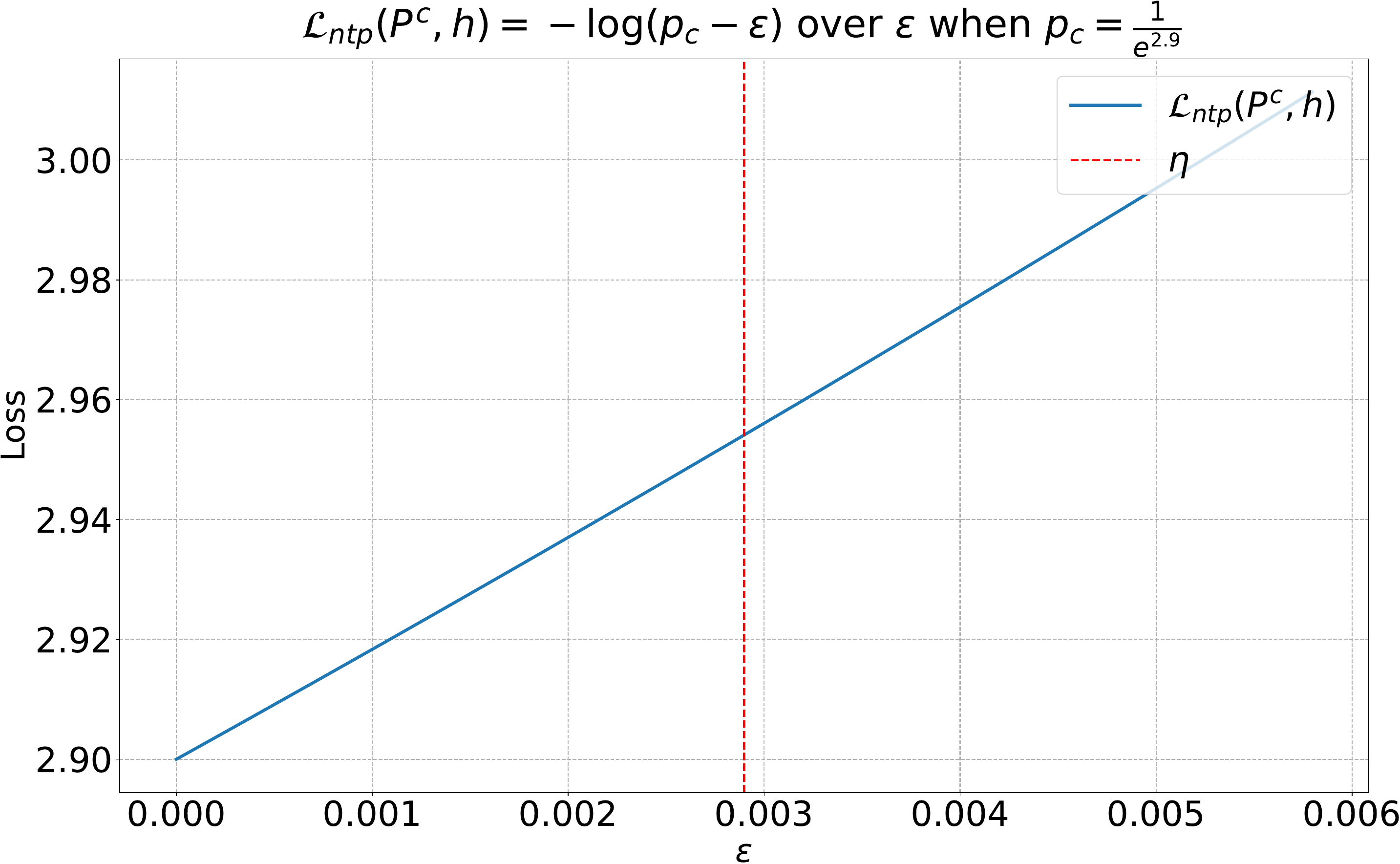}
    \caption*{(b)}
  \end{subfigure}

  \caption{Visualization of (a) $\mathcal{L}_{ntp}(P^m, h^*) - \mathcal{L}_{ntp}(P^m, h)$ and (b) $\mathcal{L}_{ntp}(P^c, h)$ with $\alpha=0.55,k=1, \mathcal{L}_{ntp}(P^c, h^*)=2.9, \mathcal{L}_{ntp}(P^n, h^*)=2.8$.}
  \label{fig:app:ml}
\end{figure*}

To illustrate our theory's explanatory power concerning multilingual models, we have plotted the scenario where $h^*$ is influenced by $P^n$ under the conditions $\alpha=0.55$, $k=1$, $\mathcal{L}_{ntp}(P^c, h^*)=2.9$, and $\mathcal{L}_{ntp}(P^n, h^*)=2.8$, as shown in Figure~\ref{fig:app:ml}. This setup simulates a model trained on a roughly 1:1 multilingual corpus, where the capacity of one language is affected by the data from another language. As can be observed from the figure, the impact on $p_c$ does not exceed $2\eta = 2(\alpha p_c - (1-\alpha)p_n)$, which translates to an increase of no more than 0.1 in $\mathcal{L}_{ntp}(P^n, h)$. This finding strongly supports the success of multilingual models from a theoretical perspective.

\subsection{Hardware}
\label{app:hard}
We conducted the pre-training process on a server equipped with 8 NVIDIA GeForce RTX 4090 GPUs. It takes approximately 70 hours to train one model using eight 4090 GPUs, so pre-training five GPT-2 models in total requires 2,800 GPU hours. Pre-training a 2.7B model takes about a week on 4 NVIDIA A100 GPUs.

\section{Experiments in Section~\ref{sec:4}}
\label{app:exp4-detail}
\subsection{Detailed Setup for Downstream Natural Language Understanding Experiments}
\subsubsection{Datasets}
We utilize 8 commonly-used text classification benchmark: SST-2, SST-fine, 20newsgroup, CR, BBC, Balanced COPA, MRPC, WiC. The detailed information can be found in Table~\ref{tab:language-data}.

\begin{table}[h]
  \centering
  \begin{tabular}{c|c|c|c}
  \toprule
  Dataset                     & Classes & Train Size & Test Size \\ \hline
  SST-2 \citep{sst2}          & 2       & 6.92k      & 1.82k     \\
  SST-fine \citep{sstfine}    & 5       & 8.54k      & 2.21k     \\
  20newsgroup \citep{20ng}    & 20      & 11.3k      & 7.53k     \\
  CR \citep{cr}               & 2       & 3.39k      & 376       \\
  BBC \citep{bbc}             & 5       & 1.23k      & 1k        \\
  Balanced COPA \citep{bcopa} & 2       & 1k         & 500       \\
  MRPC \citep{mrpc}           & 2       & 3.67k      & 1.73k     \\
  WiC \citep{wic}             & 2       & 5.43k      & 1.4k      \\ \bottomrule
  \end{tabular}
  \caption{Details of the 8 natural language understanding dataset.}
  \label{tab:language-data}
\end{table}

\subsubsection{Prompts}
Since classification tasks can be processed as seq2seq tasks by adding prompts \citep{ilya, mathlm}, we design a unique prompt for each dataset and task. This approach transforms the inputs into a format that large language models can process. The specific designs are shown in Table \ref{tab:prompt}.

\newcolumntype{Y}{>{\centering\arraybackslash}X}
\begin{table}[h]
  \centering
  \begin{tabularx}{0.85\textwidth}{c|Y}
    \toprule
    Dataset       & Prompts                                                                                                                                                                    \\ \hline
    SST-2         & \{text\} this movie is                                                                                                                                                     \\
    SST-fine      & \{text\} this movie is                                                                                                                                                     \\
    20newsgroup   & \{text\} This article is about                                                                                                                                             \\
    CR            & \{text\} the sentiment is                                                                                                                                                  \\
    BBC           & Please classify the topic of the following news: \{text\} Answer:                                                                                                            \\
    Balanced COPA & Given the premise: \{premise\} Find the most plausible alternative for the \{question\}. Option 1: \{choice1\} Option 2: \{choice2\} Which option is more plausible? \\
    MRPC          & Sentence 1: \{text1\} Sentence 2: \{text2\} Is this a paraphrase?                                                                                                       \\
    WiC           & Task: Determine if the word \{phrase1\} has the same meaning in the two sentences below. Sentence 1: \{sentence1\} Sentence 2: \{sentence2\} Your answer:                  \\ \bottomrule
    \end{tabularx}
    \caption{Details of the prompts applied to each dataset.}
  \label{tab:prompt}
\end{table}

\subsubsection{Hyperparameters}
For all experiments in Section~\ref{sec:4}, we utilize a two-layer MLP with hidden dimension equals to feature dimension and ReLU activation function.

For all experiments shown in Table~\ref{tab:1}, we set $\gamma$ in Equation~(\ref{eq:perturb}) to be 0.01 and $\lambda$ in Equation~(\ref{eq:final}) to be 0.15. Following the setup as described by \citep{mathlm}, for each dataset, we conduct a grid search on the validation set to identify the optimal learning rate and batch size. We train for a total of ten epochs with the learning rate ranging within \{3e-4, 6e-4\} and batch size options including \{8, 12, 16, 32\}. For samples without a designated validation set, we randomly select 10\% of the training set samples to form a validation set for the purpose of selecting the best parameters.

For the experiments listed in Table \ref{tab:2}, we set the batch size to 8 and the learning rate to 6e-4 for all linear probe tasks. For all MLP probe tasks, the learning rate is set to 1e-4. Regarding $\gamma$ and $\lambda$, we conduct a grid search on the validation set to find the optimal values.

\subsection{Detailed Setup for Downstream Vision Experiments}
\subsubsection{Datasets}
We select 14 image classification datasets, which serve as a common benchmark for evaluating model performance in the vision community \citep{coop, nml}. Specific information about these 14 datasets is provided in Table \ref{tab:vision-data}.
\begin{table}[h]
  \centering
  \begin{tabular}{c|ccc}
  \toprule
  Dataset                    & Classes & Train Size & Test Size \\ \hline
  StanfordCars \citep{car}   & 196     & 8144       & 8041      \\
  Caltech101 \citep{c101}    & 102     & 3060       & 6084      \\
  CIFAR-10 \citep{c10}       & 10      & 50000      & 10000     \\
  CIFAR-100 \citep{c10}     & 100     & 50000      & 10000     \\
  DTD \citep{dtd}            & 47      & 1880       & 1880      \\
  EuroSAT \citep{sat}        & 10      & 21600      & 5400      \\
  FGVCAircraft \citep{fgvc}  & 102     & 6667       & 3333      \\
  Flowers102 \citep{f102}    & 102     & 2040       & 6149      \\
  Food101 \citep{f101}       & 101     & 75750      & 25250     \\
  OxfordPet \citep{pet}      & 37      & 3680       & 3669      \\
  PatchCamelyon \citep{pcam} & 2       & 262144     & 32768     \\
  RESISC45 \citep{r45}       & 45      & 25200      & 6300      \\
  Rendered SST2 \citep{sst2}   & 2    & 6920       & 1821      \\
  SVHN \citep{svhn}          & 10      & 73257      & 26032     \\ \bottomrule
  \end{tabular}
  \caption{Details of the 14 vision dataset.}
  \label{tab:vision-data}
\end{table}

\subsubsection{Models}
We use five pre-trained general-purpose visual backbone models that cover a variety of architectures, datasets, and training methods. Detailed information is provided in Table \ref{tab:vision-model}.
\begin{table}[h]
  \centering
  \begin{tabular}{c|ccc}
  \toprule
  Model                       & Pre-training Dataset                                                                           & Size   \\ \hline
EfficientNet-B3 \citep{efn} & \begin{tabular}[c]{@{}c@{}}ImageNet-1K \citep{in1k} \\ and JFT-300M \citep{j300m}\end{tabular} & 12.3M  \\
ResNetv2-152x2 \citep{resnet} & ImageNet-21K \citep{in21k}                                                                     & 321.7M \\
Swin-L \citep{swinl}        & ImageNet-21K                                                                                   & 196.7M \\
ConvNext-L \citep{conv}     & \begin{tabular}[c]{@{}c@{}}Laion-2B \citep{laion} \\ and ImageNet-1K\end{tabular}              & 200.1M \\
ViT-L \citep{vit}           & Laion-2B                                                                                       & 428M   \\ \hline
  \end{tabular}
  \caption{Details of the 5 vision models.}
  \label{tab:vision-model}
\end{table}

\subsubsection{Hyperparameters}
\label{app:hyper2}
In our study, similar to the approach detailed in \citep{nml}, we primarily contrast our proposed method with MLP and LP tuning. For the optimization process, we employ AdamW for fine-tuning the modules over 30 epochs, utilizing a cosine learning rate scheduler. Specifically, for LP, we configure the learning rate at 0.1 without applying any weight decay. In contrast, both the MLP tuning and our method use a more conservative learning rate of 0.001 alongside a weight decay of 1e-4.

\subsubsection{Detailed Experimental Results}
In Table~\ref{tab3}, due to space limitations, we only present the average results, while detailed results are shown in Table~\ref{tab:final}.

\begin{table}[H]
  \centering
  \resizebox{\textwidth}{!}{
    \begin{tabular}{c|ll|ll|ll}
      \toprule
      \multirow{2}{*}{Models}            & \multicolumn{2}{c|}{StanfordCars}                      & \multicolumn{2}{c|}{Caltech101}                        & \multicolumn{2}{c}{CIFAR-10}                           \\
                                         & \multicolumn{1}{c}{Linear} & \multicolumn{1}{c|}{MLP}  & \multicolumn{1}{c}{Linear} & \multicolumn{1}{c|}{MLP}  & \multicolumn{1}{c}{Linear} & \multicolumn{1}{c}{MLP}   \\ \hline
      ViT-L                              & 93.38 $\pm$ 0.76           & 94.41 $\pm$ 1.05          & 92.07 $\pm$ 1.19           & 95.20 $\pm$ 1.12           & 97.99 $\pm$ 0.95           & 98.35 $\pm$ 0.95          \\
      ViT-L + $\mathcal{L}_{gm}$         & \textbf{93.71 $\pm$ 1.37}  & \textbf{94.56 $\pm$ 1.50} & \textbf{95.01 $\pm$ 0.97}  & \textbf{95.29 $\pm$ 1.27} & \textbf{98.07 $\pm$ 0.74}  & \textbf{98.48 $\pm$ 0.60} \\ \hline
      ConvNext-L                         & 86.01 $\pm$ 1.48           & 88.68 $\pm$ 0.89          & 91.02 $\pm$ 0.79           & 94.47 $\pm$ 0.53          & 97.49 $\pm$ 1.36           & 98.09 $\pm$ 0.85          \\
      ConvNext-L+$\mathcal{L}_{gm}$      & \textbf{86.78 $\pm$ 1.32}  & \textbf{89.06 $\pm$ 1.19} & \textbf{94.11 $\pm$ 1.19}  & \textbf{94.93 $\pm$ 0.88} & \textbf{97.59 $\pm$ 0.52}  & \textbf{98.15 $\pm$ 0.71} \\ \hline
      EfficientNet-B3                    & 56.20 $\pm$ 0.54            & \textbf{58.57 $\pm$ 1.11} & 89.43 $\pm$ 0.78           & 91.22 $\pm$ 1.23          & 94.04 $\pm$ 1.19           & 95.73 $\pm$ 1.07          \\
      EfficientNet-B3+$\mathcal{L}_{gm}$ & \textbf{57.02 $\pm$ 1.28}  & 58.15 $\pm$ 1.12          & \textbf{90.25 $\pm$ 1.43}  & \textbf{91.55 $\pm$ 0.90} & \textbf{94.11 $\pm$ 0.88}  & \textbf{95.96 $\pm$ 0.86} \\ \hline
      ResNetv2-152x2                     & 56.95 $\pm$ 1.21           & \textbf{59.18 $\pm$ 1.34} & 91.40 $\pm$ 1.47            & 92.48 $\pm$ 1.30          & 96.28 $\pm$ 1.16           & 96.91 $\pm$ 0.89          \\
      ResNetv2-152x2+$\mathcal{L}_{gm}$  & \textbf{58.78 $\pm$ 1.16}  & 58.67 $\pm$ 1.27          & \textbf{93.83 $\pm$ 0.91}  & \textbf{93.95 $\pm$ 0.94} & \textbf{96.31 $\pm$ 0.85}  & \textbf{97.03 $\pm$ 0.53} \\ \hline
      Swin-L                             & 68.17 $\pm$ 0.98           & \textbf{74.11 $\pm$ 0.60} & 92.58 $\pm$ 0.95           & 94.09 $\pm$ 1.04          & 98.26 $\pm$ 0.89           & 98.61 $\pm$ 0.78          \\
      Swin-L+$\mathcal{L}_{gm}$          & \textbf{69.31 $\pm$ 1.07}  & 73.71 $\pm$ 0.94          & \textbf{93.65 $\pm$ 1.42}  & \textbf{94.62 $\pm$ 0.64} & \textbf{98.41 $\pm$ 0.91}  & \textbf{98.72 $\pm$ 1.28} \\ \bottomrule
    \end{tabular}}
  \vspace{-3mm}
\end{table}
\begin{table}[H]
  \centering
  \resizebox{\textwidth}{!}{
    \begin{tabular}{ll|ll|ll|ll}
      \hline
      \multicolumn{2}{c|}{CIFAR-100}                         & \multicolumn{2}{c|}{EuroSAT}                           & \multicolumn{2}{c|}{FGVCAircraft}                      & \multicolumn{2}{c}{OxfordPet}                          \\
      \multicolumn{1}{c}{Linear} & \multicolumn{1}{c|}{MLP}  & \multicolumn{1}{c}{Linear} & \multicolumn{1}{c|}{MLP}  & \multicolumn{1}{c}{Linear} & \multicolumn{1}{c|}{MLP}  & \multicolumn{1}{c}{Linear} & \multicolumn{1}{c}{MLP}   \\ \hline
      \textbf{88.07 $\pm$ 0.58}  & 89.49 $\pm$ 0.52          & 97.53 $\pm$ 1.13           & 97.75 $\pm$ 0.61          & 65.76 $\pm$ 0.73           & 68.43 $\pm$ 0.78          & 91.65 $\pm$ 1.18           & 93.97 $\pm$ 1.50          \\
      88.06 $\pm$ 0.93           & \textbf{89.58 $\pm$ 0.93} & \textbf{97.83 $\pm$ 0.73}  & \textbf{98.03 $\pm$ 0.60} & \textbf{66.63 $\pm$ 1.08}  & \textbf{68.67 $\pm$ 1.07} & \textbf{93.18 $\pm$ 0.81}  & \textbf{94.17 $\pm$ 1.10} \\ \hline
      \textbf{86.76 $\pm$ 0.91}  & 87.79 $\pm$ 1.27          & 95.57 $\pm$ 1.22           & 96.31 $\pm$ 1.11          & 57.18 $\pm$ 1.12           & 62.25 $\pm$ 0.66          & 94.98 $\pm$ 0.54           & 95.80 $\pm$ 0.75           \\
      86.46 $\pm$ 1.04           & \textbf{87.88 $\pm$ 1.24} & \textbf{96.05 $\pm$ 1.46}  & \textbf{96.74 $\pm$ 1.27} & \textbf{58.35 $\pm$ 0.68}  & \textbf{63.61 $\pm$ 0.83} & \textbf{95.39 $\pm$ 1.21}  & \textbf{95.99 $\pm$ 0.92} \\ \hline
      \textbf{77.34 $\pm$ 0.86}  & 80.28 $\pm$ 1.02          & 94.81 $\pm$ 1.35           & 95.90 $\pm$ 0.88           & 44.73 $\pm$ 1.31           & 46.23 $\pm$ 0.92          & 93.84 $\pm$ 1.14           & 94.79 $\pm$ 1.14          \\
      77.16 $\pm$ 1.11           & \textbf{80.47 $\pm$ 1.43} & \textbf{95.20 $\pm$ 0.52}   & \textbf{96.07 $\pm$ 1.18} & \textbf{45.33 $\pm$ 0.61}  & \textbf{47.07 $\pm$ 0.57} & \textbf{94.63 $\pm$ 1.03}  & \textbf{94.98 $\pm$ 1.14} \\ \hline
      \textbf{84.30 $\pm$ 1.18}   & \textbf{84.68 $\pm$ 1.33} & 97.12 $\pm$ 1.46           & 97.46 $\pm$ 1.28          & 42.03 $\pm$ 0.72           & 48.39 $\pm$ 0.85          & 91.93 $\pm$ 0.68           & 92.99 $\pm$ 1.40          \\
      84.28 $\pm$ 1.22           & 84.29 $\pm$ 1.38          & \textbf{97.35 $\pm$ 1.12}  & \textbf{97.59 $\pm$ 0.72} & \textbf{45.69 $\pm$ 0.80}  & \textbf{48.84 $\pm$ 0.61} & \textbf{92.61 $\pm$ 0.87}  & \textbf{93.45 $\pm$ 1.32} \\ \hline
      89.68 $\pm$ 1.33           & 90.74 $\pm$ 0.98          & \textbf{97.11 $\pm$ 0.72}  & 97.59 $\pm$ 0.63          & 54.96 $\pm$ 1.24           & \textbf{61.17 $\pm$ 1.35} & 92.17 $\pm$ 0.64           & 94.38 $\pm$ 1.21          \\
      \textbf{89.79 $\pm$ 0.52}  & \textbf{91.18 $\pm$ 1.21} & 97.09 $\pm$ 1.11           & \textbf{97.71 $\pm$ 1.08} & \textbf{56.10 $\pm$ 0.67}   & 60.99 $\pm$ 0.73          & \textbf{93.86 $\pm$ 0.85}  & \textbf{94.57 $\pm$ 1.11} \\ \hline
      \end{tabular}}
  \vspace{-3mm}
\end{table}
\begin{table}[H]
  \centering
  \resizebox{\textwidth}{!}{
    \begin{tabular}{ll|ll|ll|ll}
      \hline
      \multicolumn{2}{c|}{Food101}                           & \multicolumn{2}{c|}{Flowers102}                        & \multicolumn{2}{c|}{DTD}                               & \multicolumn{2}{c}{SVHN}                               \\
      \multicolumn{1}{c}{Linear} & \multicolumn{1}{c|}{MLP}  & \multicolumn{1}{c}{Linear} & \multicolumn{1}{c|}{MLP}  & \multicolumn{1}{c}{Linear} & \multicolumn{1}{c|}{MLP}  & \multicolumn{1}{c}{Linear} & \multicolumn{1}{c}{MLP}   \\ \hline
      90.51 $\pm$ 1.31           & 91.04 $\pm$ 1.35          & 94.04 $\pm$ 1.13           & 97.83 $\pm$ 0.74          & 80.53 $\pm$ 1.06           & 83.29 $\pm$ 0.86          & 78.82 $\pm$ 1.25           & 84.74 $\pm$ 1.08          \\
      \textbf{90.62 $\pm$ 1.37}  & \textbf{91.23 $\pm$ 0.52} & \textbf{96.67 $\pm$ 1.41}  & \textbf{98.06 $\pm$ 1.28} & \textbf{82.76 $\pm$ 0.71}  & \textbf{83.77 $\pm$ 0.98} & \textbf{79.80 $\pm$ 0.88}  & \textbf{84.59 $\pm$ 1.38} \\ \hline
      \textbf{89.09 $\pm$ 1.06}  & \textbf{90.21 $\pm$ 0.87} & 94.71 $\pm$ 1.31           & 98.78 $\pm$ 1.17          & 76.01 $\pm$ 1.03           & 78.67 $\pm$ 1.15          & 66.16 $\pm$ 0.87           & 72.76 $\pm$ 0.78          \\
      88.62 $\pm$ 0.72           & 90.10 $\pm$ 1.39          & \textbf{97.12 $\pm$ 0.95}  & \textbf{98.99 $\pm$ 0.99} & \textbf{77.92 $\pm$ 0.72}  & \textbf{80.05 $\pm$ 1.44} & \textbf{68.43 $\pm$ 1.37}  & \textbf{73.18 $\pm$ 0.67} \\ \hline
      \textbf{76.95 $\pm$ 0.82}  & \textbf{81.78 $\pm$ 0.89} & 84.19 $\pm$ 0.61           & 88.97 $\pm$ 1.32          & 69.09 $\pm$ 1.27           & 73.08 $\pm$ 1.30          & 54.26 $\pm$ 0.87           & 61.38 $\pm$ 0.82          \\
      76.35 $\pm$ 0.74           & 81.33 $\pm$ 1.45          & \textbf{86.19 $\pm$ 1.08}  & \textbf{89.42 $\pm$ 0.62} & \textbf{71.27 $\pm$ 1.24}  & \textbf{73.82 $\pm$ 1.40} & \textbf{56.74 $\pm$ 0.73}  & \textbf{63.56 $\pm$ 0.82} \\ \hline
      \textbf{84.83 $\pm$ 1.36}  & 84.15 $\pm$ 1.27          & 96.76 $\pm$ 0.88           & 98.27 $\pm$ 0.53          & 72.23 $\pm$ 0.94           & 76.11 $\pm$ 1.25          & 60.75 $\pm$ 0.89           & 64.87 $\pm$ 1.45          \\
      84.41 $\pm$ 0.88           & \textbf{84.64 $\pm$ 1.04} & \textbf{98.08 $\pm$ 1.23}  & \textbf{98.84 $\pm$ 0.82} & \textbf{74.73 $\pm$ 1.38}  & \textbf{77.12 $\pm$ 1.39} & \textbf{62.04 $\pm$ 1.01}  & \textbf{65.06 $\pm$ 0.62} \\ \hline
      90.23 $\pm$ 0.64           & 92.23 $\pm$ 0.55          & 97.28 $\pm$ 0.92           & 99.51 $\pm$ 1.17          & 75.85 $\pm$ 0.88           & 80.74 $\pm$ 1.45          & 62.77 $\pm$ 1.42           & \textbf{69.53 $\pm$ 1.22} \\
      \textbf{90.26 $\pm$ 1.14}  & \textbf{92.32 $\pm$ 1.00} & \textbf{99.12 $\pm$ 1.05}  & \textbf{99.60 $\pm$ 0.51} & \textbf{77.44 $\pm$ 1.00}  & \textbf{80.91 $\pm$ 0.83} & \textbf{64.83 $\pm$ 0.98}  & 68.97 $\pm$ 1.12          \\ \hline
      \end{tabular}}
  \vspace{-3mm}
\end{table}
\begin{table}[H]
  \centering
  \resizebox{\textwidth}{!}{
    \begin{tabular}{ll|ll|ll|ll}
      \hline
      \multicolumn{2}{c|}{resisc45}                          & \multicolumn{2}{c|}{rsst2}                             & \multicolumn{2}{c|}{pcam}                              & \multicolumn{2}{c}{Avg}                                \\
      \multicolumn{1}{c}{Linear} & \multicolumn{1}{c|}{MLP}  & \multicolumn{1}{c}{Linear} & \multicolumn{1}{c|}{MLP}  & \multicolumn{1}{c}{Linear} & \multicolumn{1}{c|}{MLP}  & \multicolumn{1}{c}{Linear} & \multicolumn{1}{c}{MLP}   \\ \hline
      95.44 $\pm$ 1.41           & 95.79 $\pm$ 0.57          & 67.65 $\pm$ 1.12           & 73.58 $\pm$ 1.31          & \textbf{82.65 $\pm$ 0.57}  & \textbf{83.92 $\pm$ 0.56} & 86.86            & 89.12          \\
      \textbf{95.73 $\pm$ 1.19}  & \textbf{95.93 $\pm$ 1.28} & \textbf{71.82 $\pm$ 0.55}  & \textbf{74.24 $\pm$ 0.76} & 82.55 $\pm$ 1.49           & 83.78 $\pm$ 0.97          & \textbf{88.03}  & \textbf{89.31} \\ \hline
      92.65 $\pm$ 1.25           & 93.09 $\pm$ 0.89          & 60.73 $\pm$ 1.30           & 66.0 $\pm$ 0.55           & 72.21 $\pm$ 0.72           & 77.08 $\pm$ 0.53          & 82.89            & 85.71          \\
      \textbf{92.93 $\pm$ 0.54}  & \textbf{93.31 $\pm$ 1.31} & \textbf{64.14 $\pm$ 0.99}  & \textbf{67.49 $\pm$ 0.63} & \textbf{73.1 $\pm$ 1.35}   & \textbf{78.31 $\pm$ 1.10} & \textbf{84.07}  & \textbf{86.27} \\ \hline
      87.19 $\pm$ 1.05           & 89.01 $\pm$ 0.70          & \textbf{50.46 $\pm$ 1.02}  & 50.74 $\pm$ 1.05          & 53.32 $\pm$ 0.59           & \textbf{51.10 $\pm$ 1.34} & 73.27           & 75.62          \\
      \textbf{87.33 $\pm$ 0.60}  & \textbf{89.17 $\pm$ 1.34} & 50.19 $\pm$ 0.74           & \textbf{51.07 $\pm$ 0.92} & \textbf{54.52 $\pm$ 1.27}  & 50.10 $\pm$ 0.92          & \textbf{74.02}  & \textbf{75.90} \\ \hline
      90.96 $\pm$ 0.81           & 91.19 $\pm$ 0.57          & 50.90 $\pm$ 0.73           & 49.91 $\pm$ 1.08          & 77.62 $\pm$ 0.85           & 77.86 $\pm$ 1.16          & 78.14           & \textbf{79.60} \\
      \textbf{91.17 $\pm$ 0.64}  & \textbf{91.36 $\pm$ 0.73} & \textbf{54.25 $\pm$ 0.69}  & \textbf{49.92 $\pm$ 1.03} & \textbf{79.44 $\pm$ 0.52}  & \textbf{78.48 $\pm$ 0.53} & \textbf{79.49}  & 79.45          \\ \hline
      92.79 $\pm$ 1.25           & 94.09 $\pm$ 0.99          & 50.96 $\pm$ 1.41           & 53.59 $\pm$ 1.20          & 77.32 $\pm$ 0.68           & 78.38 $\pm$ 1.04          & 81.43           & 84.19          \\
      \textbf{93.41 $\pm$ 1.32}  & \textbf{94.42 $\pm$ 0.96} & \textbf{53.48 $\pm$ 0.89}  & \textbf{54.96 $\pm$ 0.90} & \textbf{81.18 $\pm$ 1.34}  & \textbf{79.29 $\pm$ 1.31} & \textbf{82.70}  & \textbf{84.42} \\ \hline
      \end{tabular}}
  \caption{Detailed accuracy of 5 vision backbone models on \textbf{14} commonly-used vision datasets.}
  \label{tab:final}
\end{table}

\subsection{Runtime Analysis}
\label{app:time}
Since all our models are black-box models, we first process all samples into vector-form features and then probe them. All models described in this paper can run on a single NVIDIA RTX 4090 GPU. Extracting all these features requires a total of 10 GPU hours. Subsequently, training these Linear or MLP Probes requires approximately 200 GPU hours in total.

\section{Limitations}
\label{limitation}
Firstly, due to limitations in computational resources and costs, we pre-train only the GPT-2 124M and 2.7B model on the OpenWebText dataset. Additionally, the types of noise considered are limited to uniform and Gaussian distributions. However, based on Proposition~\ref{pro1}, we argue that training GPT-2 on the Synthetic OpenWebText dataset is sufficient to uncover the essence of the issue, as Proposition~\ref{pro1} makes no assumptions about data distribution or model architecture.

Secondly, on the theoretical front, we consider neural networks as black boxes and focus on analyzing the properties of global minima. Due to limited mathematical skills, we do not delve into the dynamical aspects to specifically examine how random noise within datasets influences model gradients, nor do we explore the differences between global and local minima obtained through stochastic gradient descent. However, experimental results indicate that neural networks trained with stochastic gradient descent do not suffer from significant disturbances.

\end{document}